\documentclass{article}

\usepackage{times}
\usepackage{graphicx} 
\usepackage{subfigure} 

\usepackage{natbib}

\usepackage{algorithm}
\usepackage{algorithmic}

\usepackage{hyperref}

\usepackage{url}            

\usepackage{graphicx} 
\usepackage{caption}
\usepackage{subfigure} 

\usepackage{mathtools}
\usepackage{amsmath,amssymb,amsthm}

\usepackage[accepted]{icml2018}

\def\tx{\tilde{x}}
\def\hy{\hat{y}}

\def\R{{\mathbb R}}
\def\CE{\mathbf{E}}

\newtheorem{theorem}{Theorem}
\newtheorem{collary}{Corollary}
\newtheorem{proposition}{Proposition}
\newtheorem{lemma}{Lemma}
\newtheorem{assumption}{Assumption}
\newtheorem{remark}{Remark}

\graphicspath{{images/}} 

\icmltitlerunning{Lightweight Stochastic Optimization 
for Minimizing Finite Sums with Infinite Data}

\begin{document} 

\twocolumn[
\icmltitle{Lightweight Stochastic Optimization \\
for Minimizing Finite Sums with Infinite Data}

\begin{icmlauthorlist}
\icmlauthor{Shuai Zheng}{ust}
\icmlauthor{James T. Kwok}{ust}
\end{icmlauthorlist}

\icmlaffiliation{ust}{Department of Computer Science and Engineering, Hong Kong University of
Science and Technology, Clear Water Bay, Hong Kong}
\icmlcorrespondingauthor{Shuai Zheng}{szhengac@cse.ust.hk}

\icmlkeywords{machine learning, stochastic optimization}

\vskip 0.3in
]
\printAffiliationsAndNotice{} 

\begin{abstract} 
Variance reduction
has been commonly used in stochastic optimization.
It relies crucially on the assumption that the data set is finite.
However, 
when the data are imputed with random noise as in data augmentation,
the perturbed data set becomes essentially infinite. 
Recently, the stochastic MISO (S-MISO) 
algorithm is introduced to address this 
expected risk minimization problem.
Though it
converges faster than SGD, 
a significant amount of memory
is required. 
In this paper, we propose two SGD-like algorithms for expected risk minimization with random perturbation, namely,
stochastic sample average gradient (SSAG) and stochastic SAGA (S-SAGA).
The memory cost of SSAG does not depend on the sample size, while that of S-SAGA is the same as those of variance reduction methods on unperturbed data. 
Theoretical analysis and 
experimental results on logistic regression and AUC maximization
show that SSAG has faster convergence rate than SGD with comparable space requirement, while S-SAGA outperforms S-MISO in terms of both iteration complexity and 
storage. 


\end{abstract} 


\section{Introduction}

Machine learning tasks are often cast as optimization problems with some data
distributions. 
In regularized risk minimization with $n$ training samples, one minimizes:
\begin{eqnarray} \label{eq:basic_erm}
\min_{\theta} 
\frac{1}{n}\sum_{i=1}^n\ell_i(\theta) + g(x),
\end{eqnarray}
where 
$\theta$ is the model parameter, 
$\ell_i$ is the loss 
due to sample $i$, 
and $g$ is a regularizer.
In this paper, we assume that $\ell_i$ and $g$ are smooth and convex.
Stochastic gradient descent (SGD) \cite{robbins1951stochastic} 
and its variants
\cite{nemirovski2009robust,xiao2010dual,duchi2011adaptive,bottou2016optimization}
are flexible, scalable, and
widely used for this problem.
However, SGD suffers from large variance due to sampling noise. 
To alleviate this problem, 
the stepsize
has to be decreasing,
which slows convergence. 

By exploiting the finite-sum structure
in (\ref{eq:basic_erm}),
a class of 
variance-reduced
stochastic optimization methods 
have been
proposed
recently
\cite{roux2012stochastic,johnson2013accelerating,shalev-shwartz-13,Mairal2013,defazio-14,defazio-14b}. 
Based on the use of control variates
\cite{fishman1996monte},
they construct different
approximations to  the true gradient so that its
variance decreases as the 
optimal solution is
approached. 


In order to capture more variations in the data distribution, it is effective 
to obtain more training data by injecting random noise to the data samples \cite{decoste-02,maaten2013learning,paulin2014transformation}. 
Theoretically, it has been shown that random noise improves generalization \cite{wager2014altitude}.
In addition, artificially corrupting the training data has a wide range of applications in machine learning. 
For example, additive Gaussian noise can be used in image denoising \cite{vincent2010stacked} and provides a form of $\ell_2$-type regularization \cite{bishop1995training}; 
dropout noise serves as adaptive regularization that is useful in stabilizing predictions \cite{maaten2013learning} 
and selecting discriminative but rare features \cite{wager2013dropout}; and Poisson noise is
of interest to count features as in document classification \cite{maaten2013learning}.  

With the addition of noise perturbations, 
(\ref{eq:basic_erm}) becomes 
the following expected risk minimization problem:
\begin{equation} \label{eq:p2}
\min_{\theta} 
\frac{1}{n}\sum_{i=1}^n\CE_{\xi_i}[\ell_i(\theta; \xi_i)] + g(x),
\end{equation} 
where $\xi_i$ is the random noise injected to function $\ell_i$, and $\CE_{\xi}$
denotes expectation w.r.t. $\xi_i$.  Because of the expectation,
the perturbed data can be considered as infinite, and
the finite data set assumption 
in variance reduction methods
is violated.
In this case, each function in problem (\ref{eq:p2}) can only be accessed via a
stochastic first-order oracle, and the main optimization tool is SGD.

Despite its importance,
expected risk minimization 
has received very little attention.
One very recent work for this is the stochastic MISO (S-MISO) 
\cite{bietti2016stochastic}. While it
converges 
faster 
than SGD,
S-MISO requires 
$O(nd)$ space, where $d$ is the feature  dimensionality.
This significantly limits its applicability to big data problems. 
The N-SAGA algorithm \cite{hofmann2015variance} can also be used on 
problems
with infinite data.  However, its asymptotic error is nonzero.

In this paper, we focus on the linear model.
By exploiting the linear structure, we propose two SGD-like variants
with low memory costs:
stochastic sample average gradient
(SSAG) and stochastic SAGA (S-SAGA).
In particular, the memory cost of SSAG does not depend on the sample size $n$, while
S-SAGA has a memory requirement of $O(n)$, which matches the stochastic variance reduction methods 
on unperturbed data
\cite{roux2012stochastic,shalev-shwartz-13,defazio-14,defazio-14b}.
Similar to S-MISO, the proposed algorithms have faster convergence than SGD.
Moreover, the convergence rate of S-SAGA depends on a 
constant 
that is typically smaller 
than that of S-MISO.
Experimental results on logistic regression and AUC maximization with dropout noise
demonstrate the efficiency of the proposed algorithms.

{\bf Notations}. 
For a vector $x$,
$\|x\|=\sqrt{\sum_i x_i^2}$ is its $\ell_2$-norm.
For two vectors $x$ and $y$,  $x^Ty$ denotes its dot
product.



\section{Related Work}\label{sec:review}



In this paper, we consider the linear model.
Given
samples $\{x_1,\dots,x_n\}$, with each  $x_i \in \R^d$,
the regularized risk minimization problem
in (\ref{eq:basic_erm})
can be written as:
\begin{eqnarray} \label{eq:erm_problem}
\min_{\theta} 
\frac{1}{n}\sum_{i=1}^n\phi_i(x_i^T\theta) + g(\theta),
\end{eqnarray}
where 
$\hy_i\equiv x_i^T\theta$ is the prediction on sample $i$,
and $\phi_i$ is a loss.
For example, logistic regression corresponds to $\phi_i(\hy_i) = \log(1 + \exp(-y_i\hy_i))$, where $\{y_1,\dots, y_n\}$ are the training labels; and linear regression corresponds to $\phi_i(\hy_i) = (y_i - \hy_i)^2$. 

\subsection{Learning with Injected Noise}\label{sec:inj_review}

To make the predictor 
robust,
one can 
inject i.i.d. random noise $\xi_i$ to each sample $x_i$ 
\cite{maaten2013learning}. 
Let the perturbed sample be $\hat{x}_i \equiv \psi(x_i, \xi_i)$.
The following types of noise have been popularly used:
(i) additive noise \cite{bishop1995training, wager2013dropout}: 
$\hat{x}= x +
\xi$, where $\xi$ comes from a zero-mean distribution such as the normal or Poisson distribution;
and
(ii) dropout noise \cite{srivastava2014dropout}: 
$\hat{x}= \xi \circ
x$, where 
$\circ$ denotes the element-wise product,
$\xi \in \{0, 1/(1 - p)\}^d$, $p$ is the dropout probability,
and each component of 
$\xi$
is an independent draw from a scaled Bernoulli$(1-p)$ random variable.  
With random perturbations,
(\ref{eq:erm_problem}) becomes the following expected risk minimization
problem:
\begin{equation} \label{eq:em_problem}
\min_{\theta} F(\theta)  \equiv \frac{1}{n}\sum_{i=1}^n\CE_{\xi_i}[\phi_i(
\hat{x}_i^T\theta)] + g(\theta).
\end{equation}
As the objective contains an expectation, 
computing the gradient is infeasible as infinite samples are needed. 
As an approximation,
SGD uses the
gradient 
from a single sample.
However, this has large variance.

In this paper, we make the following assumption
on 
$f_i(\theta; \xi_i) \equiv \phi_i(\hat{x}_i^T\theta) + g(\theta)$
in (\ref{eq:em_problem}).
Note that this implies $\phi_i(\hat{x}_i^T\theta)$ and $F$ are also $L$-smooth.
\begin{assumption} \label{assumption:lipschitz}
Each $f_i(\theta; \xi_i)$ is $L$-smooth
w.r.t. $\theta$, i.e.,
there exists constant $L$ such that $\|\nabla f_i(\theta; \xi_i) - \nabla f_i(\theta'; \xi_i)\| \leq L\|\theta - \theta'\|, \forall \theta, \theta'$.
\end{assumption}


\subsection{Variance Reduction}

In stochastic optimization,
control variates have been commonly used 
to reduce the variance of stochastic gradients
\cite{fishman1996monte}.
In general,
given a random variable $X$
and another 
highly correlated
random variable $Y$, 
a variance-reduced estimate of $\CE X$ 
can be obtained 
as
\begin{eqnarray} \label{eq:control_vr}
X - Y + \CE Y.
\end{eqnarray}
In stochastic optimization on problem (\ref{eq:erm_problem}),
the gradient 
$\phi'_i(x_i^T\theta)x_i$ 
of the loss evaluated on sample
$x_i$ 
is taken as $X$. 
When the training set is finite, various algorithms have been recently proposed  so that
$Y$ is strongly correlated with $\phi'_i(x_i^T\theta)x_i$
and $\CE Y$ can be easily evaluated.
Examples include stochastic average gradient (SAG)
\cite{roux2012stochastic},
MISO
\cite{Mairal2013}, 
stochastic variance reduced gradient (SVRG)
\cite{johnson2013accelerating},
Finito \cite{defazio-14b}, SAGA 
\cite{defazio-14}, and stochastic dual coordinate ascent (SDCA) \cite{shalev-shwartz-13}.


However, 
with the expectation 
in (\ref{eq:em_problem}),
the full gradient (i.e., $\CE Y$ in (\ref{eq:control_vr})) cannot be evaluated, and
variance reduction can no longer be used.
Very recently, the stochastic MISO (S-MISO) algorithm \cite{bietti2016stochastic} is
proposed for solving (\ref{eq:em_problem}).  
Its convergence rate outperforms that of SGD by having a smaller multiplicative constant. 
However, S-MISO
requires an additional
$O(nd)$ space,
which prevents its use on large data sets. 


\section{Sample Average Gradient}
\label{sec:soda}


Let the iterate 
at iteration $t$
be $\theta_{t-1}$.
To approximate the gradient
$\nabla F(\theta_{t-1})$
in (\ref{eq:em_problem}),
SGD uses the
gradient $g_t = \phi_{i_t}'(\hat{x}_{i_t}^T\theta_{t-1})\hat{x}_{i_t} + \nabla
g(\theta_{t-1})$ evaluated 
on a single sample
$\hat{x}_{i_t}$,
where $i_t$ is sampled uniformly from $[n]\equiv  \{1, 2, \dots, n\}$. 
The variance  of
$g_t$ is usually assumed to be bounded  by a constant, as
\begin{equation} \label{eq:var_s}
\CE\|g_t - \nabla F(\theta_{t-1})\|^2 \leq \sigma_s^2,
\end{equation} 
where the expectation is taken w.r.t. both the random index $i_t$ and perturbation $\xi_t$
at iteration $t$. Note that the gradient of regularizer $g$ does not contribute to the variance. 



\subsection{Exploiting the Model Structure}

\subsubsection{Stochastic Sample-Average Gradient (SSAG)}
\label{sec:i}



At iteration $t$,
the stochastic gradient of the loss $\phi_i(\hat{x}_i^T\theta)$  for sample
$\hat{x}_{i_t}$
is
$\phi'(\hat{x}_{i_t}^T\theta)\hat{x}_{i_t}$.
Thus, the gradient direction is determined by 
$\hat{x}_{i_t}$, while
parameter $\theta$
only affects
its scale.
With this observation, 
we consider using
$a_t\hat{x}_{i_t}$
as a control variate  for
$\phi'(\hat{x}_{i_t}^T\theta)\hat{x}_{i_t}$,
where $a_t$ may depend on past information but not
on $\hat{x}_{i_t}$.
Note that 
the gradient component $\nabla g(\theta)$ 
is deterministic, and does not contribute to the construction of control variate. 
Using (\ref{eq:control_vr}),
the resultant gradient estimator  is:
\begin{equation} \label{eq:z}
z_t = (\phi_{i_t}'(\hat{x}_{i_t} ^T\theta_{t-1}) - a_t)\hat{x}_{i_t} + a_t\tx_t + \nabla g(\theta_{t-1}),
\end{equation} 
where $\tx_t$ is an estimate of $\CE[\hat{x}_{i_t}]$. 
For example,
$\tx_t$ can be 
defined
as 
\begin{equation} \label{eq:upd}
\tx_t = \left(1 - \frac{1}{t} \right)\tx_{t-1} + \frac{1}{t}\hat{x}_{i_t}, 
\end{equation}
so that $\tx_t$ 
can be incrementally updated as 
$\hat{x}_{i_t}$'s
are sampled. 
As $\hat{x}_{i_t}$'s are i.i.d.,
by the law of large number,
the sample average $\tx_t$ 
converges to the expected value
$\CE[\hat{x}_{i_t}]$.

The following shows that $z_t$ 
in (\ref{eq:z})
is a biased estimator of 
the gradient $\nabla F(\theta_{t-1})$. 
As $\tx_t$ converges to $\CE[\hat{x}_{i_t}]$, $z_t$ is still asymptotically unbiased.
\begin{proposition} \label{prop:asym_unbiased_est}
$\CE[z_t] = \nabla F(\theta_{t-1})+ a_t\left(1 - \frac{1}{t}\right)(\tx_{t-1}
-\CE[\hat{x}_{i_t}])$.
\end{proposition} 

Note that $\CE[\hat{x}_{i_t}] =
\frac{1}{n}\sum_{i=1}^n\CE_{\xi_i}[\hat{x}_i]$, where $\CE_{\xi_i}$ denotes the expectation w.r.t. $\xi_i$.
We assume that   each
$\CE_{\xi_i}[\hat{x}_i]$
can be 
easily computed. 
This is the case, for example, when 
the noise is
dropout noise or additive zero-mean noise,
and 
$\CE_{\xi_i}[\hat{x}_i]
= x_i$ \cite{maaten2013learning}. 
This  suggests 
replacing $\tx_t$ in (\ref{eq:z}) by $\tx \equiv \frac{1}{n}\sum_{i=1}^n\CE_{\xi_i}[\hat{x}_i]$
(which is equal to $\CE[\hat{x}_{i_t}]$),
leading to the estimator:
\begin{eqnarray} \label{eq:ssag1_est}
v_t = (\phi_{i_t}'(\hat{x}_{i_t} ^T\theta_{t-1}) - a_t)\hat{x}_{i_t} + a_t\tx + \nabla g(\theta_{t-1}).
\end{eqnarray}
The following shows that
$v_t$
is unbiased, 
and also provides an upper bound of its variance.
\begin{proposition} \label{prop:em_var}
$\CE[v_t] = \nabla F(\theta_{t-1})$, and
$\CE[\|v_t - \nabla F(\theta_{t-1})\|^2]
\leq \CE[(\phi_{i_t}'(\hat{x}_{i_t}^T\theta_{t-1}) - a_t)^2\|\hat{x}_{i_t}\|^2] \label{em_var}$.
The bound is minimized when  
\begin{eqnarray} \label{em_opt_a}
a_t = a_t^* \equiv \frac{\CE[\phi'(\hat{x}^T\theta_{t-1})\|\hat{x}\|^2]}{\CE[\|\hat{x}\|^2]}.
\end{eqnarray}
\end{proposition}
For 
dropout noise and other additive noise with
known variance,
one can compute $\CE_{\xi_i}\|\hat{x}_i\|^2$ for each $i\in [n]$, and then average
to obtain $\CE[\|\hat{x}\|^2]$.
However, 
evaluating the expectation in the numerator of (\ref{em_opt_a})
is infeasible.

Instead, we define $a_t$ as
\begin{equation} \label{eq:at}
a_t = \tilde{a}_t/s_t
\end{equation} 
for $t \geq 1$, and
approximate the expectations in the numerator and denominator 
by moving averages:
\begin{eqnarray*} \label{eq:app_2}
\tilde{a}_{t+1}  &= & (1 - \beta_t)\tilde{a}_{t} + \beta_{t}\phi_{i_t}'(\hat{x}_{i_t}^T\theta_{t-1})\|\hat{x}_{i_t}\|^2, \\
s_{t+1} &= &  (1 - \beta_t) s_{t} + \beta_{t}\|\hat{x}_{i_t}\|^2.
\end{eqnarray*}
We initialize $a_1 = \tilde{a}_1 = s_1 = 0$, and set $\beta_t \in [0, 1)$. 

The resulting
algorithm, called stochastic sample-average gradient (SSAG),
is shown in
Algorithm~\ref{alg:ssag1}.
Compared to S-MISO \cite{bietti2016stochastic}, SSAG 
is more computationally efficient. 
It
does not require an extra $O(nd)$
memory, and only requires one single gradient evaluation (step~6) in each
iteration. 

\begin{algorithm}[ht]
\caption{Stochastic sample-average gradient (SSAG).}
   \label{alg:ssag1}
\begin{algorithmic}[1]
\STATE {\bfseries Input:} $\eta_t > 0$, $\beta_t \in [0, 1)$.
   \STATE {\bfseries initialize} 
	$\theta_0$;
      $\tx \leftarrow \frac{1}{n}\sum_{i=1}^n\CE_{\xi_i}[\hat{x}_i]$;
$a_1 \leftarrow 0$; $\tilde{a}_1 \leftarrow 0$; $s_1 \leftarrow 0$
   \FOR{$t=1, 2, \dots $}
     \STATE{draw sample index $i_t$ and random perturbation $\xi_t$}
     \STATE{$\hat{x}_{i_t} \leftarrow \psi(x_{i_t}, \xi_t)$}
     \STATE{$d_t \leftarrow  \phi_{i_t}'(\hat{x}_{i_t} ^T\theta_{t-1})$}
    \STATE{$v_t \leftarrow (d_t - a_t)\hat{x}_{i_t}  + a_t\tx + \nabla g(\theta_{t-1})$}
    \STATE{$\theta_t  \leftarrow \theta_{t-1} - \eta_t v_t$}
        \STATE{$\tilde{a}_{t+1} \leftarrow   (1 - \beta_t)\tilde{a}_{t} + \beta_t d_t\|\hat{x}_{i_t}\|^2$ \\   
               $s_{t+1} \leftarrow   (1 - \beta_t)s_{t} + \beta_t \|\hat{x}_{i_t}\|^2$ \\
               $a_{t+1} \leftarrow \tilde{a}_{t+1}/s_{t+1}$}
   \ENDFOR
\end{algorithmic}
\end{algorithm}

The following Proposition shows that $a_t$ in (\ref{eq:at})
is asymptotically optimal 
for appropriate choices of 
$\eta_t$ and $\beta_t$.

\begin{proposition} \label{lemma:opt_beta}
If 
(i) $\CE[\phi_{i_t}'(\hat{x}_{i_t}^T\theta_{t-1})^2\|\hat{x}_{i_t}\|^4] < \infty \text{ and } \CE[\|\hat{x}_{i_t}\|^4] < \infty$;
(ii) $\|v_t\| < \infty$; 
(iii) 
$\eta_t \rightarrow 0, \hspace{.1in} \sum_{t}\eta_t = \infty, \hspace{.1in} \sum_t\eta_t^2 <
\infty$;
(iv) $\beta_t \rightarrow 0, \hspace{.1in} \sum_{t}\beta_t = \infty, \hspace{.1in}
\sum_t\beta_t^2 < \infty$; and
(v) $\eta_t/\beta_t \rightarrow 0$,
then
\[a_t \rightarrow a_t^* \hspace{.1in} w.p. 1. \]
\end{proposition} 
A simple choice 
is:
$\eta_t = O(1/t^{c_1}), \beta_t = O(1/t^{c_2})$,
where 
$1/2 < c_2 < c_1 \leq 1$.
The following Proposition quantifies the convergence of $s_ta_t$ to $s_ta_*^t$.
In particular,
when $c_1 = 1$, the asymptotic bound 
in (\ref{eq:bnd})
is minimized when $c_2 = 2/3$. 
\begin{proposition} \label{lemma:opt_beta_conv}
With assumptions (i)-(v) in Proposition~\ref{lemma:opt_beta}, 
$\eta_t = O(1/t^{c_1})$, and $\beta_t = O(1/t^{c_2})$,
we have 
\begin{equation} \label{eq:bnd} 
\CE[s_t^2(a_t - a_t^*)^2]
\leq O\left(\max\left\{\frac{1}{t^{c_2}}, \frac{1}{t^{2(c_1 - c_2)}}\right\}\right).
\end{equation} 
\end{proposition} 



\subsubsection{Stochastic SAGA (S-SAGA)}



Recall that in 
(\ref{eq:ssag1_est}),
$\phi_{i_t}'(\hat{x}_{i_t} ^T\theta_{t-1}) \hat{x}_{i_t}$ plays the role of 
$X$
in (\ref{eq:control_vr}), and
$a_t\hat{x}_{i_t}$ plays the role of $Y$.
However, the corresponding
$X$ and $Y$ 
in (\ref{eq:control_vr})
can be negatively correlated in some iterations.
This is partly because 
$a_t$ in (\ref{eq:ssag1_est}) does not depend on $\hat{x}_{i_t}$,
though
$a_t\hat{x}_{i_t}$ serves as a control variate for
$\phi_{i_t}'(\hat{x}_{i_t} ^T\theta_{t-1}) \hat{x}_{i_t}$.
Thus,
it is better 
for each sample 
$\hat{x}_{i}$
to have its own
scaling factor,
leading to the
estimator:
\begin{eqnarray} \label{eq:eq:ssag2_est}
v_t = (\phi_{i_t}'(\hat{x}_{i_t} ^T\theta_{t-1}) - a_{i_t})\hat{x}_{i_t} + m_{t-1} + \nabla g(\theta_{t-1}), 
\end{eqnarray}
where $m_{t-1} 
=\CE[a_{i_t}\hat{x}_{i_t}]
= \frac{1}{n}\sum_{i=1}^na_{i}
\CE_{\xi_i}[\hat{x}_i]$.
Note that
$m_t$ can be updated sequentially as:
\[ m_t = m_{t-1} + \frac{1}{n}(\phi_{i_t}'(\hat{x}_{i_t} ^T\theta_{t-1}) -
a_{i_t})\CE_{\xi_t}[\hat{x}_{i_t}]. \] 
Besides, (\ref{eq:eq:ssag2_est}) reduces to the SAGA estimator \cite{defazio-14} when the
random noise is removed.
The whole procedure, which will be called 
stochastic SAGA (S-SAGA),
is shown in Algorithm~\ref{alg:ssag2}. 


\begin{algorithm}[ht]
\caption{Stochastic SAGA (S-SAGA).}
   \label{alg:ssag2}
\begin{algorithmic}[1]
   \STATE {\bfseries Input:} $\eta_t > 0$.
   \STATE {\bfseries initialize} 
	$\theta_0$;
          $\bar{x}_i \leftarrow \CE_{\xi_i}[\hat{x}_i]$ and $a_i \leftarrow \phi_{i}'(\hat{x}_i^T, \theta_0)$ for all $i \in [n]$;
    $m_0 = \frac{1}{n}\sum_{i=1}^na_i\bar{x}_i$
   \FOR{$t=1, 2, \dots $}
      \STATE{draw sample index $i_t$ and random perturbation $\xi_t$}
     \STATE{$\hat{x}_{i_t} \leftarrow \psi(x_{i_t}, \xi_t)$}
     \STATE{$d_t \leftarrow  \phi_{i_t}'(\hat{x}_{i_t} ^T\theta_{t-1})$}
    \STATE{$v_t \leftarrow (d_t - a_{i_t})\hat{x}_{i_t}  + m_{t-1} + \nabla g(\theta_{t-1})$}
    \STATE{$\theta_t  \leftarrow \theta_{t-1} - \eta_tv_t$}
    \STATE{$m_t \leftarrow  m_{t-1} + \frac{1}{n}(d_t - a_{i_t})\bar{x}_{i_t} $}
     \STATE{$a_{i_t} \leftarrow  d_t$}
   \ENDFOR
\end{algorithmic}
\end{algorithm}


S-SAGA
needs an additional
$O(nd)$ space
for $\{a_1,\dots,a_n\}$ and
$\{\CE_{\xi_1}[\hat{x}_1],\dots,
\{\CE_{\xi_n}[\hat{x}_n]\}$.
However, as discussed in Section~\ref{sec:i},
$\CE_{\xi_i}[\hat{x}_i] = x_i$
for many types of noise.
Hence,
$\CE_{\xi_i}[\hat{x}_i]$'s
do not need to be explicitly stored,
and the additional space is reduced to $O(n)$. This is significantly
smaller than that of S-MISO, which always requires $O(nd)$ additional space. 


\subsection{Convergence Analysis}
\label{sec:convg}

In this section, we provide convergence results for SSAG and S-SAGA on problem (\ref{eq:em_problem}). 

\subsubsection{SSAG}

For SSAG,
we make the following additional assumptions.

\begin{assumption} 
$F$ is $\mu$-strongly convex, i.e., 
$F(\theta') \geq F(\theta) + \langle \nabla F(\theta), \theta' - \theta\rangle +
\frac{\mu}{2}\|\theta' - \theta\|^2, \forall \theta, \theta'$.
\end{assumption} 
\begin{assumption} 
$\CE[(\phi_{i_t}'(\hat{x}_{i_t}^T\theta_{t-1}) - a_t)^2\|\hat{x}_{i_t}\|^2] \leq \sigma_a^2$
for all $t$.
\end{assumption} 

Let the minimizer of (\ref{eq:em_problem}) be $\theta_*$. 
The following Theorem shows that
SSAG has
$O(1/t)$ convergence rate,
which is similar to SGD \cite{bottou2016optimization}. 
\begin{theorem} \label{em_conv}
Assume that
$\eta_t = c/(\gamma + t)$
for some $c > 1/\mu$ and $\gamma > 0$ such that $\eta_1 \leq 1/L$. 
For the
$\{\theta_t\}$ sequence generated from
SSAG, we have
\begin{eqnarray} \label{eq:em_conv_vrsgd}
\CE[F(\theta_t)] - F(\theta_*) \leq \frac{\nu_1}{\gamma + t + 1},
\end{eqnarray}
where
$\nu_1 \equiv \max\left\{\frac{c^2L\sigma_a^2}{2(c\mu-1)}, (\gamma + 1)C_1\right\}$, and $C_1 = F(\theta_0) - F(\theta_*)$.
\end{theorem}
Note that this $\eta_t$ also satisfies the conditions in Proposition~\ref{lemma:opt_beta}.
The condition $c > 1/\mu$ is crucial to obtaining the $O(1/t)$ rate.
It has been observed that underestimating $c$ can make convergence extremely slow \cite{nemirovski2009robust}. 
When the model is $\ell_2$-regularized, 
$\mu$ can be estimated by the corresponding regularization parameter.

\begin{collary} \label{collary:complexity_ssag1}
To ensure that $\CE[F(\theta_t)] - F(\theta_*) \leq \epsilon$, 
SSAG 
has a time complexity of 
$O(n +\kappa C_1/\epsilon + \sigma_a^2\kappa^2/\epsilon)$,
where $\kappa \equiv L/\mu$
is the condition number.
\end{collary}
The $O(n)$ term is due to initialization of $m_0$ and amortized over multiple data passes. 
In contrast, the time complexity for SGD is $O(\kappa C_1/\epsilon + \sigma_s^2\kappa^2/\epsilon)$,
where
$\sigma_s^2$ is defined 
in (\ref{eq:var_s})
\cite{bottou2016optimization}.  
To compare $\sigma_s^2$ 
with $\sigma_a^2$, 
we assume that the perturbed samples have finite variance
$\sigma_x^2$:
\[\CE[\|\hat{x} - \CE[\hat{x}]\|^2] = \sigma_x^2.\]
The variance of the SGD estimator $g_t$ can be bounded as
 \begin{eqnarray*} \label{e_sgd_var}
\lefteqn{\CE[\|g_t - \nabla F(\theta_{t-1})\|^2} \nonumber\\ 
& = &\CE[\|\phi_{i_t}'(\hat{x}_{i_t}^T\theta_{t-1})\hat{x}_{i_t} - \CE[\phi_{i_t}'(\hat{x}_{i_t}^T\theta_{t-1})\hat{x}_{i_t}]\|^2] \\
& \leq & 3\CE[\|\phi_{i_t}'(\hat{x}_{i_t}^T\theta_{t-1})\hat{x}_{i_t} - a_t\hat{x}_{i_t}\|^2] \\
&& + 3\CE[\|a_t\hat{x}_{i_t} - a_t\CE[\hat{x}_{i_t}]\|^2\\
&&  + 3\CE[\|a_t\CE[\hat{x}_{i_t}] - \CE[\phi_{i_t}'(\hat{x}_{i_t}^T\theta_{t-1})\hat{x}_{i_t}]\|^2] \\
&\lessapprox& 3\sigma_a^2 + 3a_t^2\sigma_x^2.
\end{eqnarray*}
Thus, the gradient variance of SGD has  two terms, one involving
$\sigma_a^2$ and the other involving $a_t\sigma_x^2$.  In particular, if the derivative 
$\phi_{i_t}'(\hat{x}_{i_t}^T\theta_{t-1})$
is constant,  then $\sigma_a^2 = 0$, and only the perturbed sample variance $\sigma_x^2$
contributes to the gradient variance of SGD. 
For a large class of functions including the logistic loss and smoothed hinge loss, 
$\phi_{i_t}'(\hat{x}_{i_t}^T\theta_{t-1})$
is nearly constant in some regions. 
In this case, we have $a_t^2\sigma_x^2 \approx \sigma_s^2$. 

\subsubsection{S-SAGA}

Besides Assumption~\ref{assumption:lipschitz},
we assume the following:
\begin{assumption} 
Each $f_i(\theta; \xi_i)$ is $\mu$-strongly convex, i.e., 
$f_i(\theta'; \xi_i) \geq f_i(\theta; \xi_i) + \langle \nabla f_i(\theta; \xi_i),
\theta' - \theta\rangle + \frac{\mu}{2}\|\theta' - \theta\|^2, \forall \theta, \theta'$.
\end{assumption} 
\begin{assumption} 
For all $t$, 
$\frac{1}{n}\sum_{i=1}^n\!\!\CE_{\xi_i, \xi_i'}[(\phi_{i}'(\hat{x}_i'^T \theta_{t-1}) \!-\! \phi_{i}'(\hat{x}_{i}^T
\theta_{t-1}))^2\|\hat{x}_{i}\|^2]$ $\leq \sigma_c^2$,
where $\hat{x}_i = \psi(x_i, \xi_i)$, $\hat{x}_i' = \psi(x_i, \xi_i')$, and
$\xi_i'$ 
is another randomly sampled noise 
for $x_i$.
\end{assumption} 


\begin{theorem} \label{em_conv_ssaga}
Assume that 
$\eta_t = c/(\gamma + t)$
for some $c > 1/\mu$ and $\gamma > 0$ such that $\eta_1 \leq 1/(3(\mu n + L))$. 
For the $\{\theta_t\}$ sequence generated from S-SAGA,  we have
\begin{eqnarray} \label{eq:em_conv_pvrsgd}
\CE[\|\theta_t - \theta_*\|^2] \leq \frac{\nu_2}{\gamma + t + 1},
\end{eqnarray}
where
$\nu_2 \equiv \max\left(\frac{4c^2\sigma_c^2}{c\mu - 1}, (\gamma + 1)C_2\right)$,  and
$C_2 \equiv \|\theta_0 - \theta_*\|^2 
+ \frac{2n}{3(\mu n + L)}[F(\theta_0) - F(\theta_*)]$.
\end{theorem}
Thus, S-SAGA has a convergence rate of $O(\sigma_c^2\kappa^2/t)$. In comparison,
the convergence rate of SGD is $O(\sigma_s^2\kappa^2/t)$.
Note that $\sigma_s^2$ in (\ref{eq:var_s}) includes variance due to data sampling, while $\sigma_c^2$ above only considers
that due to noise. 
Since data sampling induces a much larger variation than that from perturbing the same sample,
typically we have $\sigma_c^2 \ll \sigma_s^2$,
and thus S-SAGA has faster convergence.

S-MISO considers the variance  of the difference in gradients due to noise:
\[\frac{1}{n}\sum_{i=1}^n\CE_{\xi,\xi'}[\|\phi_{i}'(\hat{x}_i'^T \theta_{t-1})\hat{x}_{i}' - \phi_{i}'(\hat{x}_{i}^T \theta_{t-1})\hat{x}_{i}\|^2] \leq \sigma_m^2,\]
and 
its convergence rate is $O(\sigma_m^2\kappa^2/t)$. The bounds for $\sigma_m^2$ and
$\sigma_c^2$ are similar in form. However,
$\sigma_c^2$ can be small when the 
difference $\phi'(\hat{x}_i^T\theta) - \phi'(\hat{x}_i'^T\theta)$ is small, while it is not the
case for $\sigma_m^2$. In particular, when $\phi'(\hat{x}^T\theta)$ is a constant
regardless of random perturbations, $\sigma_c^2 = 0$.


The following Corollary considers the time complexity of S-SAGA.
\begin{collary} \label{collary:complexity_ssag2}
To ensure that $\CE[F(\theta_t)] - F(\theta_*) \leq \epsilon$, 
S-SAGA   
has a time complexity of 
$O((n + \kappa)C_2/\epsilon + \sigma^2_c\kappa^2/\epsilon)$.
\end{collary}

\begin{remark}
In \cite{bietti2016stochastic}, additional speedup can be achieved by
first running the algorithm with a constant stepsize 
for a few epochs, and then applying the decreasing stepsize.  
This trick is not used here.
If incorporated,
it can be shown that the $C_2/\epsilon$ term 
will be improved
to $\log(C_2/\bar{\epsilon})$. 
\end{remark}

A summary of the convergence results is shown in Table~\ref{convg_com}. As can be
seen, by exploiting the linear model structure, SSAG has a smaller variance
constant 
than 
SGD 
($\sigma_a^2$ vs $\sigma_s^2$) 
while having comparable space requirement.  
S-SAGA  improves over S-MISO and achieves gains both in terms of iteration complexity and storage. 
\begin{table}[h] 
\vspace{-.1in}
\caption{Iteration complexity and extra storage of different methods for solving
optimization problem~(\ref{eq:em_problem}). For simplicity of comparison, we drop
the constant $C$. }
\label{convg_com}
\begin{center}
\begin{tabular}{ccc}
\hline
& iteration complexity & space \\ \hline
SGD  &  $O(\kappa/\epsilon + \sigma_s^2\kappa^2/\epsilon)$  & $O(d)$      \\
S-MISO  &  $O((n + \kappa)/\epsilon + \sigma_m^2\kappa^2/\epsilon)$  & $O(nd)$       \\
SSAG  &  $O(n + \kappa/\epsilon + \sigma_a^2\kappa^2/\epsilon)$  & $O(d)$       \\
S-SAGA  &  $O((n+\kappa)/\epsilon + \sigma_c^2\kappa^2/\epsilon)$  & $O(n)$  \\
\hline
\end{tabular}
\end{center}
\vspace{-.1in}
\end{table} 

\begin{figure*}[t]
\begin{center}
\subfigure[{\em avazu-app} ($\lambda = 10^{-6}$).]{\includegraphics[width=.65\columnwidth, height=.42\columnwidth]{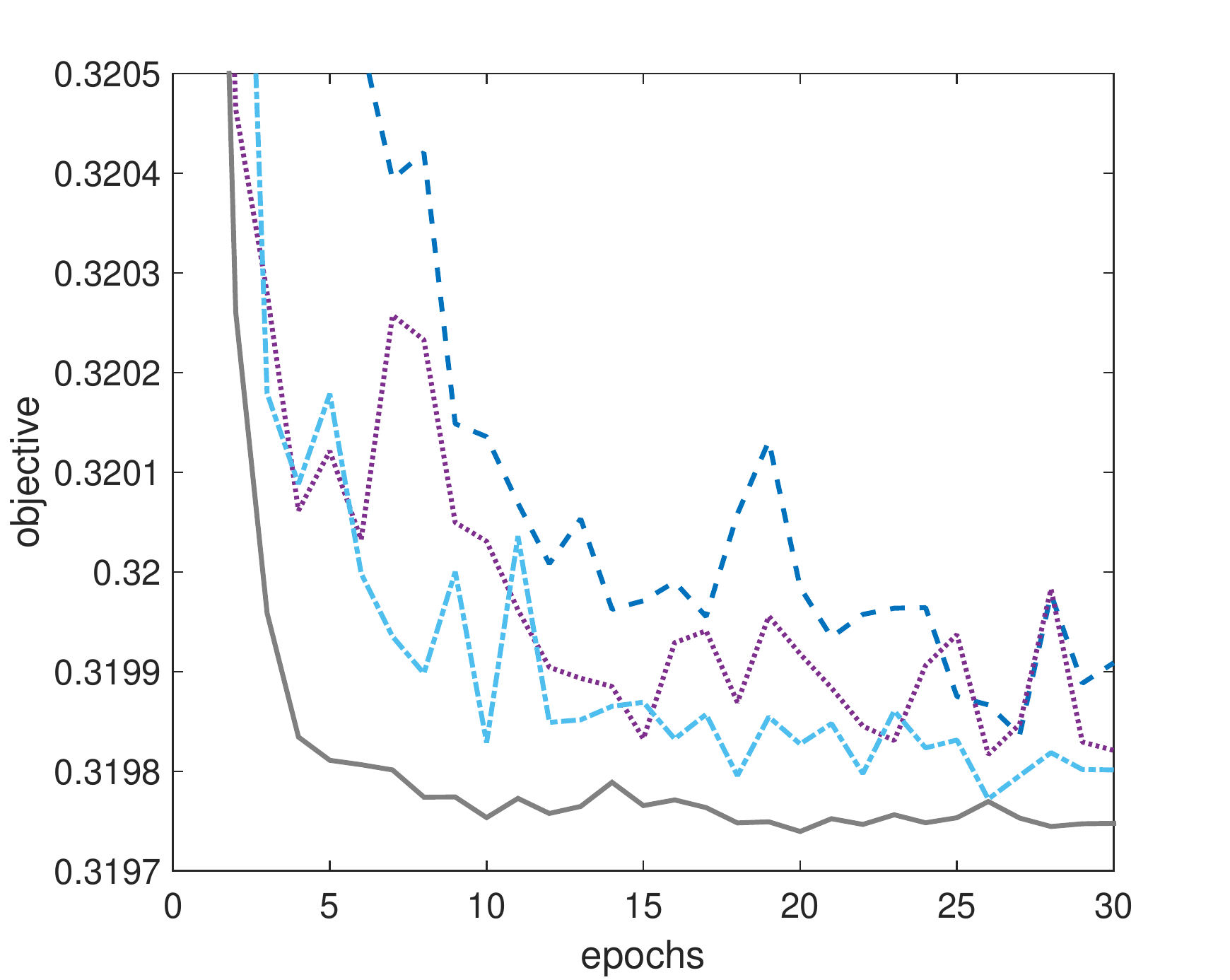}}
\subfigure[{\em avazu-app} ($\lambda = 10^{-7}$).]{\includegraphics[width=.65\columnwidth, height=.42\columnwidth]{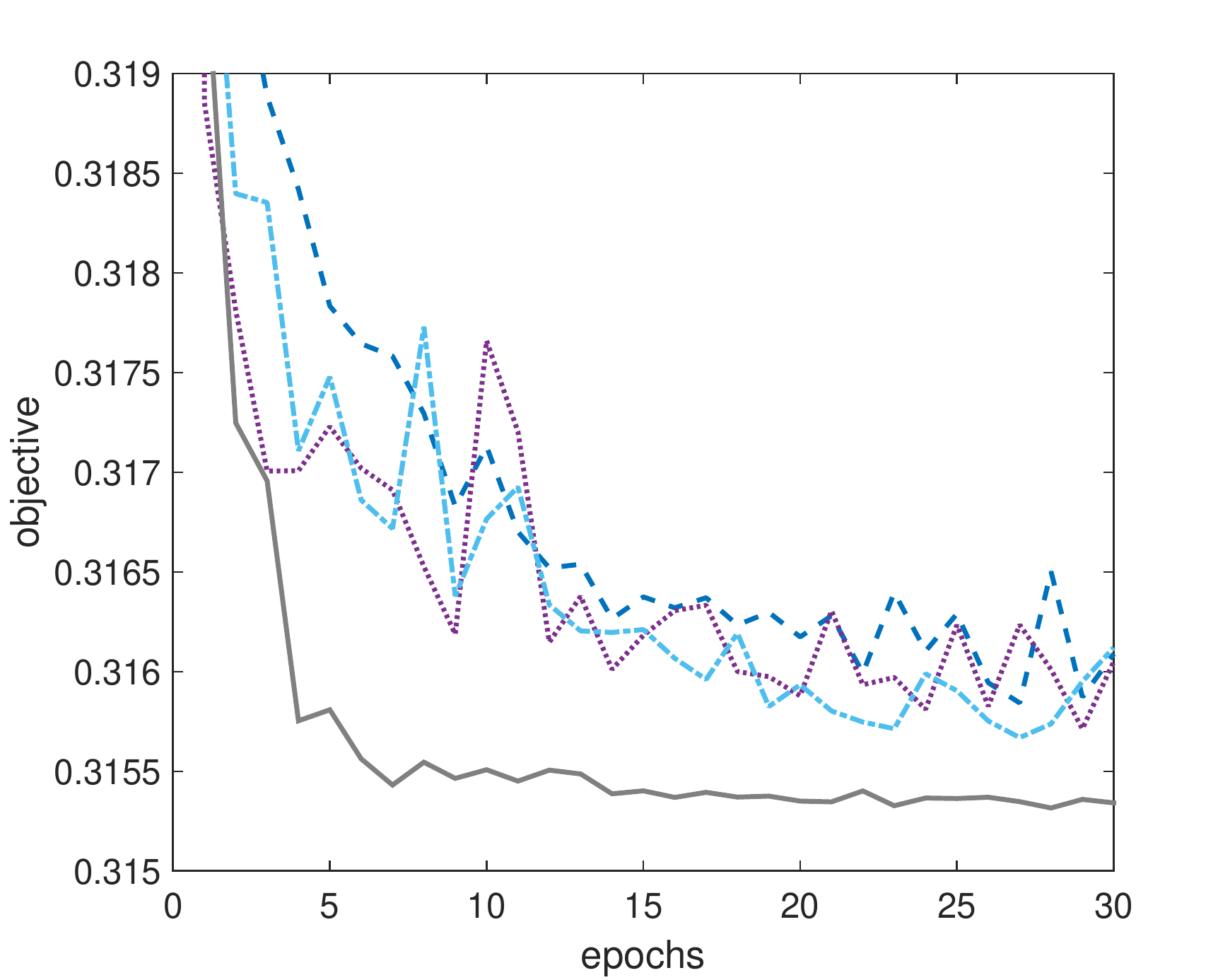}}
\subfigure[{\em avazu-app} ($\lambda = 10^{-8}$).]{\includegraphics[width=.65\columnwidth, height=.42\columnwidth]{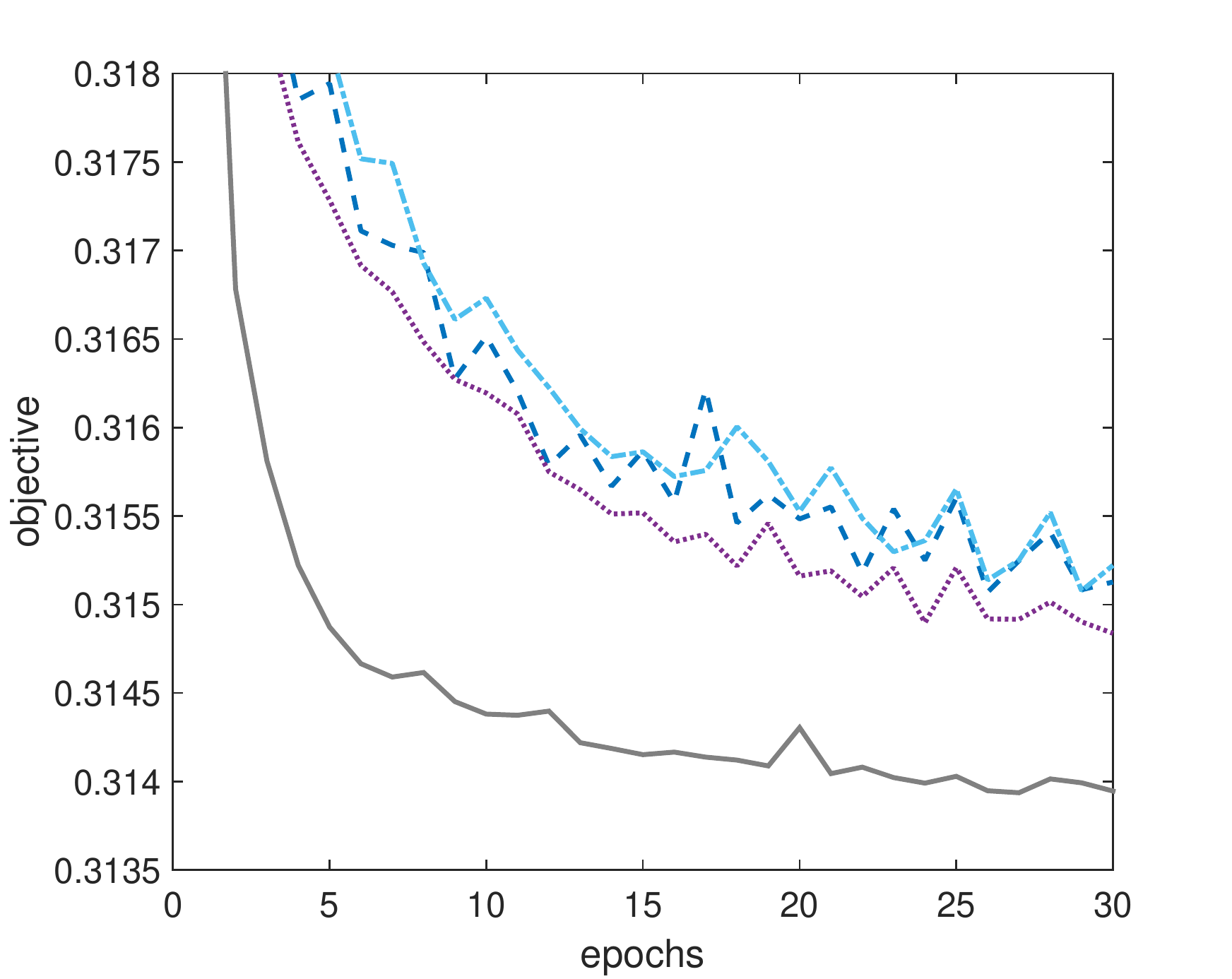}}
\vspace{-.12in}
\\
\subfigure[{\em kddb} ($\lambda = 10^{-6}$).]{\includegraphics[width=.65\columnwidth, height=.42\columnwidth]{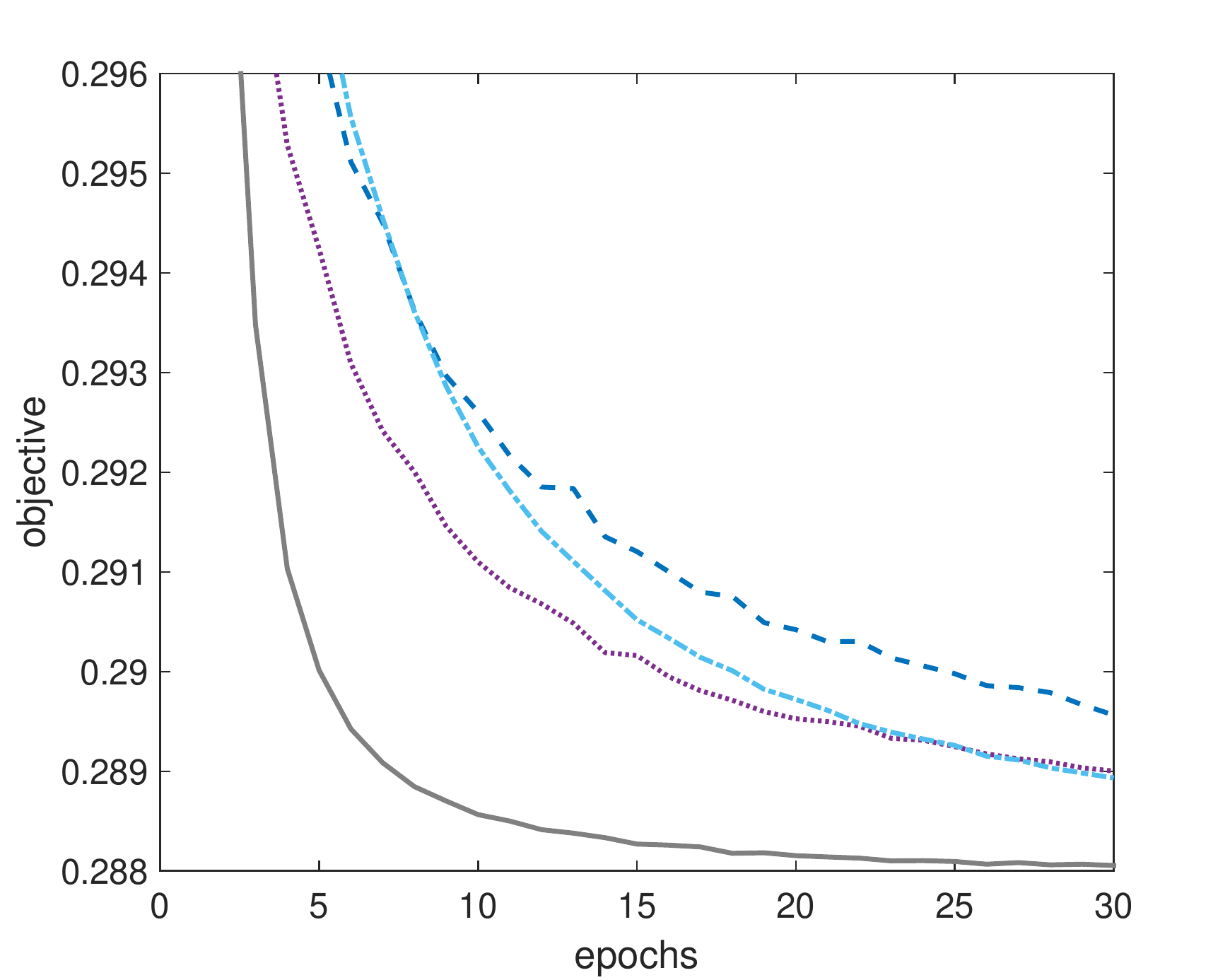}}
\subfigure[{\em kddb} ($\lambda = 10^{-7}$).]{\includegraphics[width=.65\columnwidth, height=.42\columnwidth]{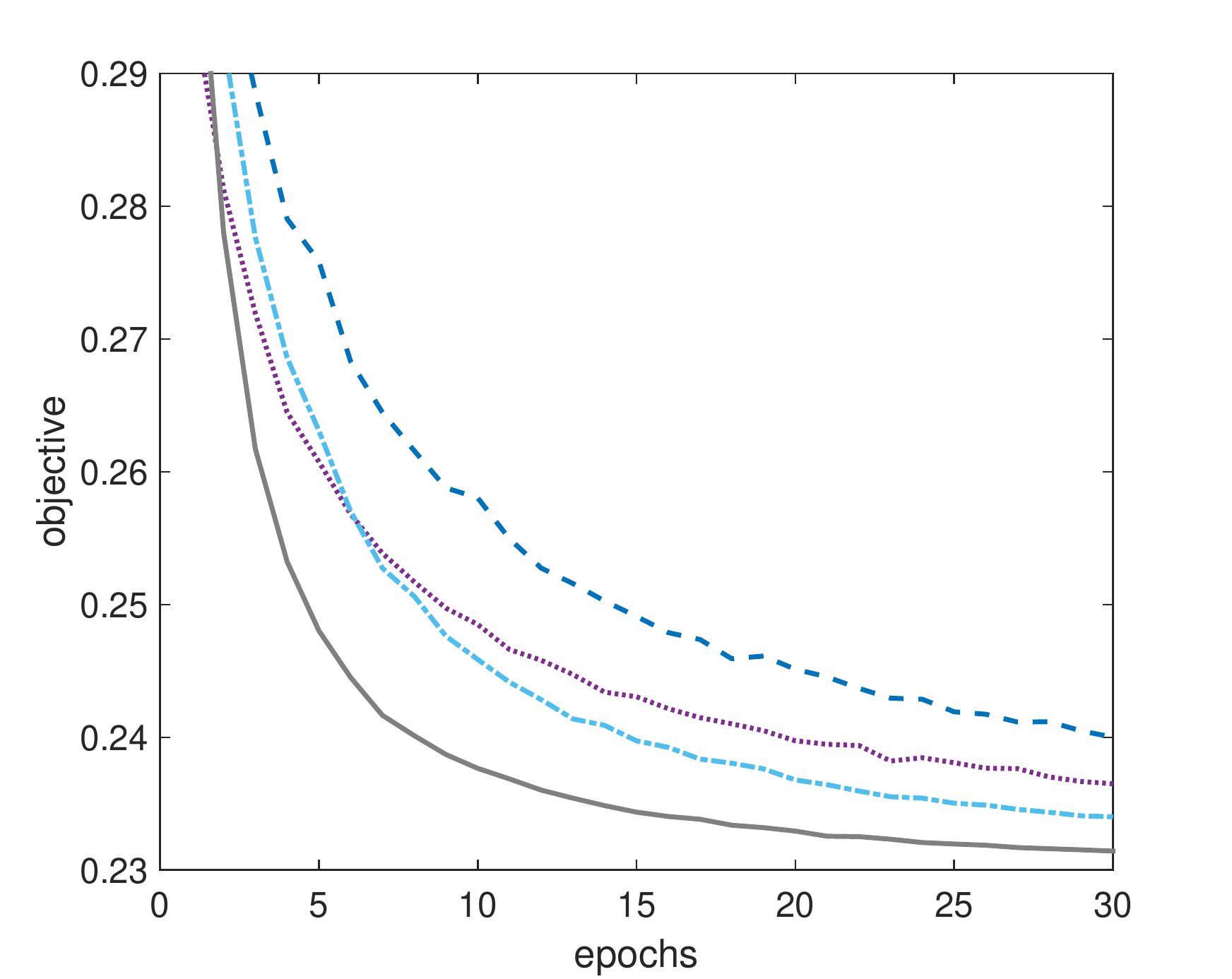}}
\subfigure[{\em kddb} ($\lambda = 10^{-8}$).]{\includegraphics[width=.65\columnwidth, height=.42\columnwidth]{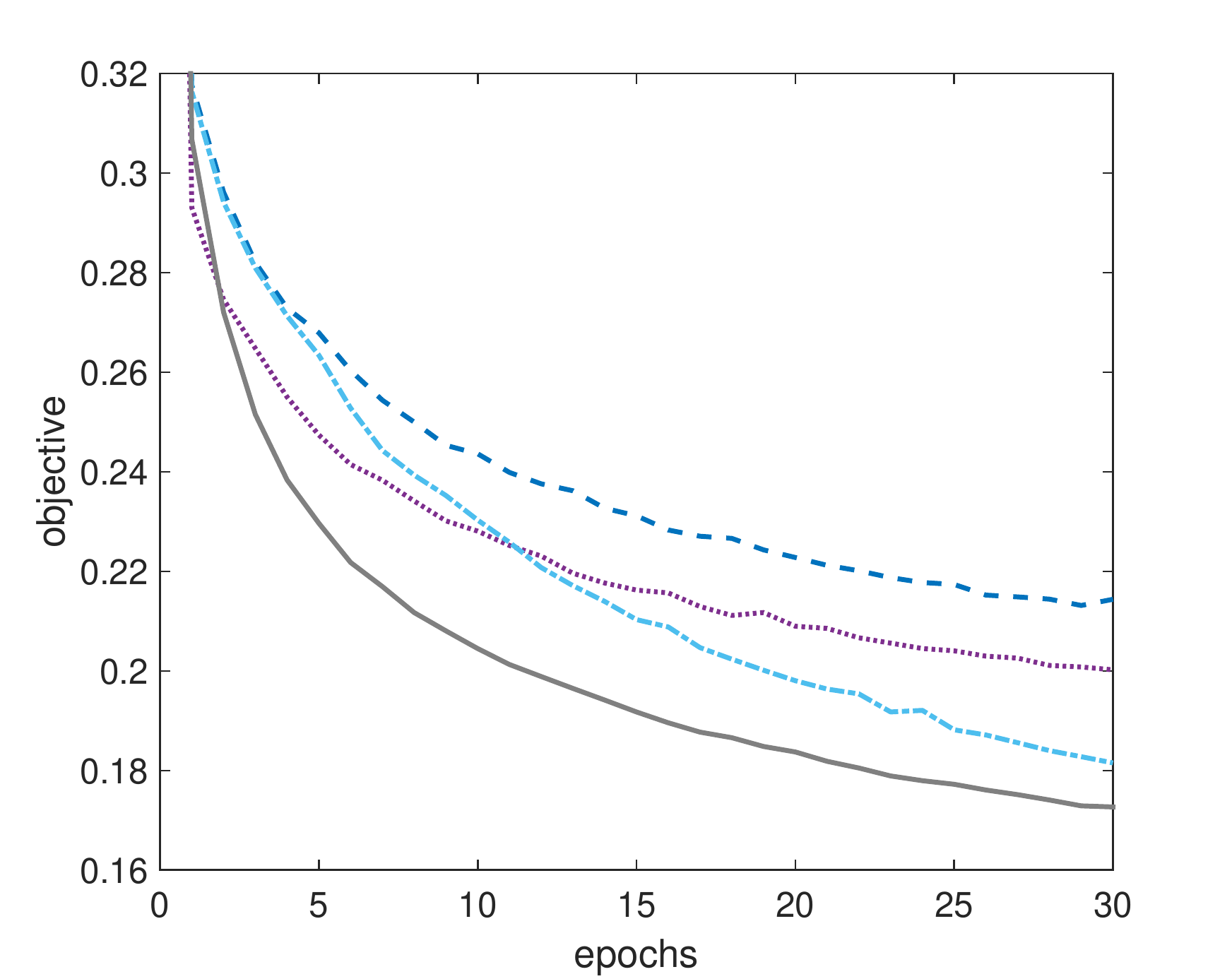}}
\setcounter{subfigure}{0}
\end{center}
\begin{center}
\includegraphics[width=3in]{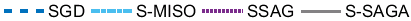}
\end{center}
\vskip -.16in
\caption{Convergence with the number of epochs (logistic regression with dropout).
The regularization parameter $\lambda$ of $\ell_2$ regularizer is varied from $10^{-6}$ to $10^{-8}$. The dropout rate is fixed to $0.3$. }
\label{classification_dropout_ni}
\end{figure*}

\subsection{Acceleration By Iterate Averaging}
\label{sec:ia}

The complexity bounds in Corollaries~\ref{collary:complexity_ssag1} and \ref{collary:complexity_ssag2} depend quadratically on the condition
number $\kappa$. This may be problematic on ill-conditioned problems. 
To alleviate this problem, one can use iterate averaging
\cite{bietti2016stochastic},
which outputs 
\begin{equation} \label{c}
\bar{\theta}_T \equiv \frac{2}{T(2\gamma + T - 1)}\sum^{T-1}_{t=0}(\gamma + t)\theta_t,
\end{equation} 
where $T$ is the total number of iterations.
It can be easily seen that 
(\ref{c}) can be efficiently implemented in an online fashion without the need for storing
$\theta_t$'s:
\[\bar{\theta}_t = (1 - \rho_t)\bar{\theta}_{t-1} + \rho_t\theta_{t-1},\]
where 
$\rho_t = \frac{2(\gamma + t - 1)}{t(2\gamma + t - 1)}$ and $\bar{\theta}_0 = 0$.
As in 
\cite{bietti2016stochastic},
the following shows that 
the $\kappa^2$ dependence in both SSAG and S-SAGA
(Corollaries~\ref{collary:complexity_ssag1} and \ref{collary:complexity_ssag2}) can be reduced
to
$\kappa$.

\begin{theorem} \label{em_a_conv}
Assume that
$\eta_t = 2/(\mu(\gamma + t))$ and
$\gamma > 0$ such that $\eta_1 \leq 1/(2L)$. 
For the $\{\theta_t\}$ sequence generated from SSAG,  we have
\begin{eqnarray*} \label{eq:em_a_conv_vrsgd}
\lefteqn{\CE[F(\bar{\theta}_T)] - F(\theta_*)}\\
& \leq & \frac{\mu\gamma(\gamma - 1)}{T(2\gamma + T-1)}\|\theta_0 - \theta_*\|^2  + \frac{4\sigma_a^2}{\mu(2\gamma + T - 1)}.
\end{eqnarray*}
\end{theorem}

The stepsize $\eta_t = 2/(\mu(\gamma + t))$ and condition $\eta_1 \leq 1/2L$ together
implies that $\gamma = O(\kappa)$. 
Thus, when
$T \ll \gamma$, the first term, which depends on 
$\|\theta_0 - \theta_*\|^2$, decays as $1/T$, which is no better than
(\ref{eq:em_conv_vrsgd}). On the other hand, if $T \gg \gamma$, the first term decays at
a faster 
$\kappa/T^2$
rate.

\begin{collary} \label{collary:complexity_ia_ssag1}
When $T \gg \gamma$, to ensure that
$\CE[F(\bar{\theta}_T)] - F(\theta_*) \leq \epsilon$, 
SSAG with iterate averaging 
has a time complexity of 
$O(n + \sqrt{\kappa C_3}/\sqrt{\epsilon} + 
\sigma^2_a\kappa/\epsilon)$, where $C_3 = \|\theta_0 - \theta_*\|^2$.
\end{collary}


Similarly, 
we have the following
for S-SAGA.
\begin{theorem} \label{em_a_conv_ssaga}
Assume that 
$\eta_t = c/(\gamma + t)$ for some $c > 1/\mu$ and
$\gamma > 0$ such that $\eta_1 \leq 1/(7(\mu n + L))$.
For the $\{\theta_t\}$ sequence generated from S-SAGA,  we have
\begin{eqnarray} \label{eq:em_a_conv_vrsaga}
\lefteqn{\CE[F(\bar{\theta}_T) - F(\theta_*)]} \nonumber\\
& \leq & \frac{\mu\gamma(\gamma - 1)}{2T(2\gamma + T - 1)}C_4  + \frac{32\sigma_c^2}{\mu(2\gamma + T - 1)},
\end{eqnarray}
where 
$C_4 \equiv 3\|\theta_0 - \theta_*\|^2  
+ \frac{4n}{7(\mu n + L)}[F(\theta_0) - F(\theta_*)]$.
\end{theorem}
The condition $\eta_1 \leq 1/(7(\mu n + L))$ is satisfied when $\gamma = O(n + \kappa)$. 
Thus, the second term in $C_4$ is scaled by $4n/(7(\mu n +L)) = O(n/(\mu \gamma))$. 
These implies that the first term in
(\ref{eq:em_a_conv_vrsaga}) decays as $n/T$ when $T \ll \gamma$. On the other hand, when $T \gg \gamma$, the first term decays as $n(n + \kappa)/T^2$. 
Thus, iterate averaging does not provide S-SAGA with much acceleration as compared to SSAG. 

The following Corollary considers the case where $n = O(\kappa)$ \cite{johnson2013accelerating}.
\begin{collary}  \label{collary:complexity_ia_ssag2}
Assume that $n = O(\kappa)$. 
When $T \gg \gamma$,
to ensure that
$\CE[F(\bar{\theta}_T)] - F(\theta_*) \leq \epsilon$, 
S-SAGA with iterate averaging 
has a time complexity of 
$O(n + \sqrt{(n + \kappa)C_4}/\sqrt{\epsilon} + 
\sigma^2_c\kappa/\epsilon)$.
\end{collary}

\begin{figure*}[ht]
\begin{center}
\subfigure[{\em avazu-app} ($\lambda = 10^{-6}$).]{\includegraphics[width=.65\columnwidth, height=.42\columnwidth]{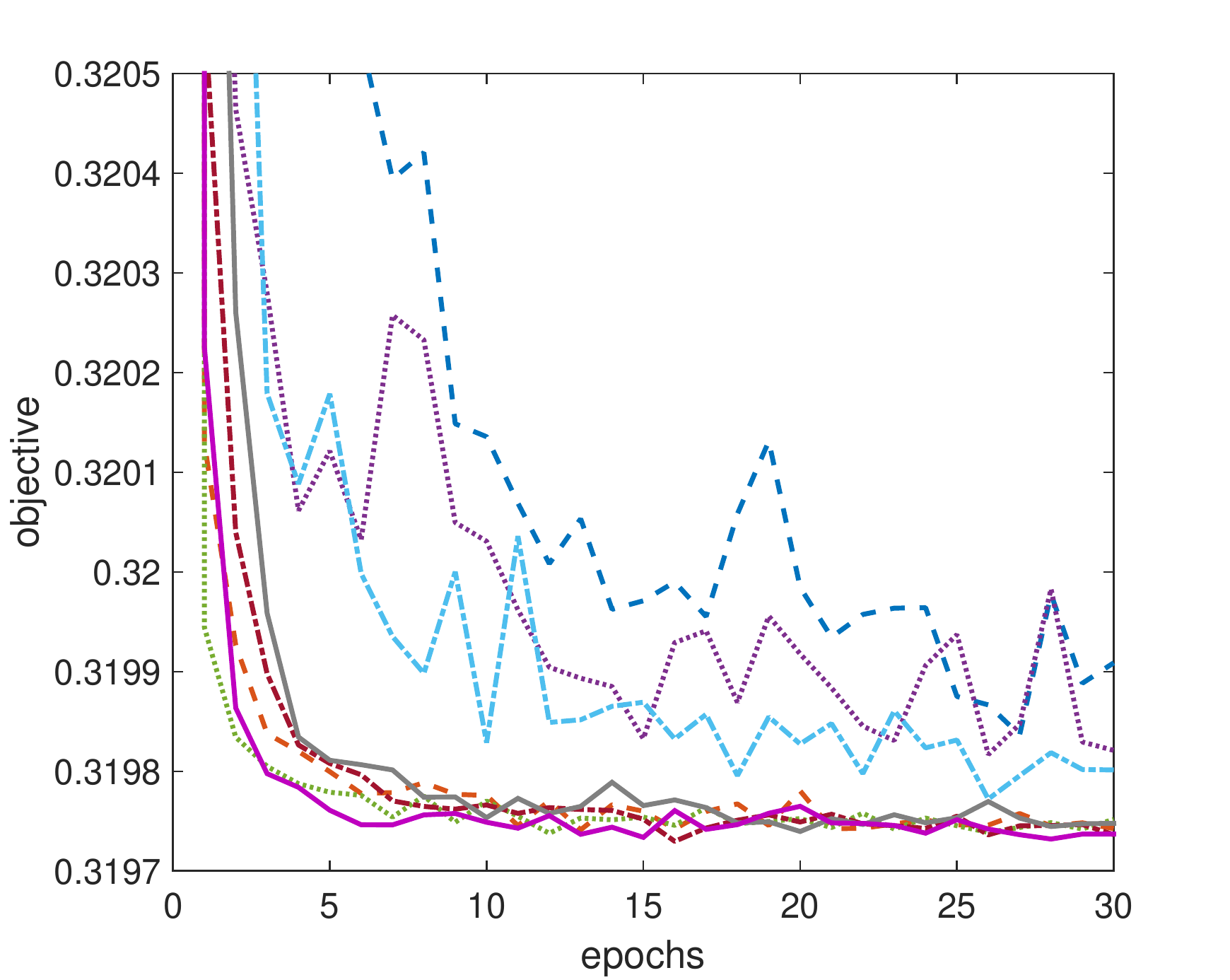}}
\subfigure[{\em avazu-app} ($\lambda = 10^{-7}$).]{\includegraphics[width=.65\columnwidth, height=.42\columnwidth]{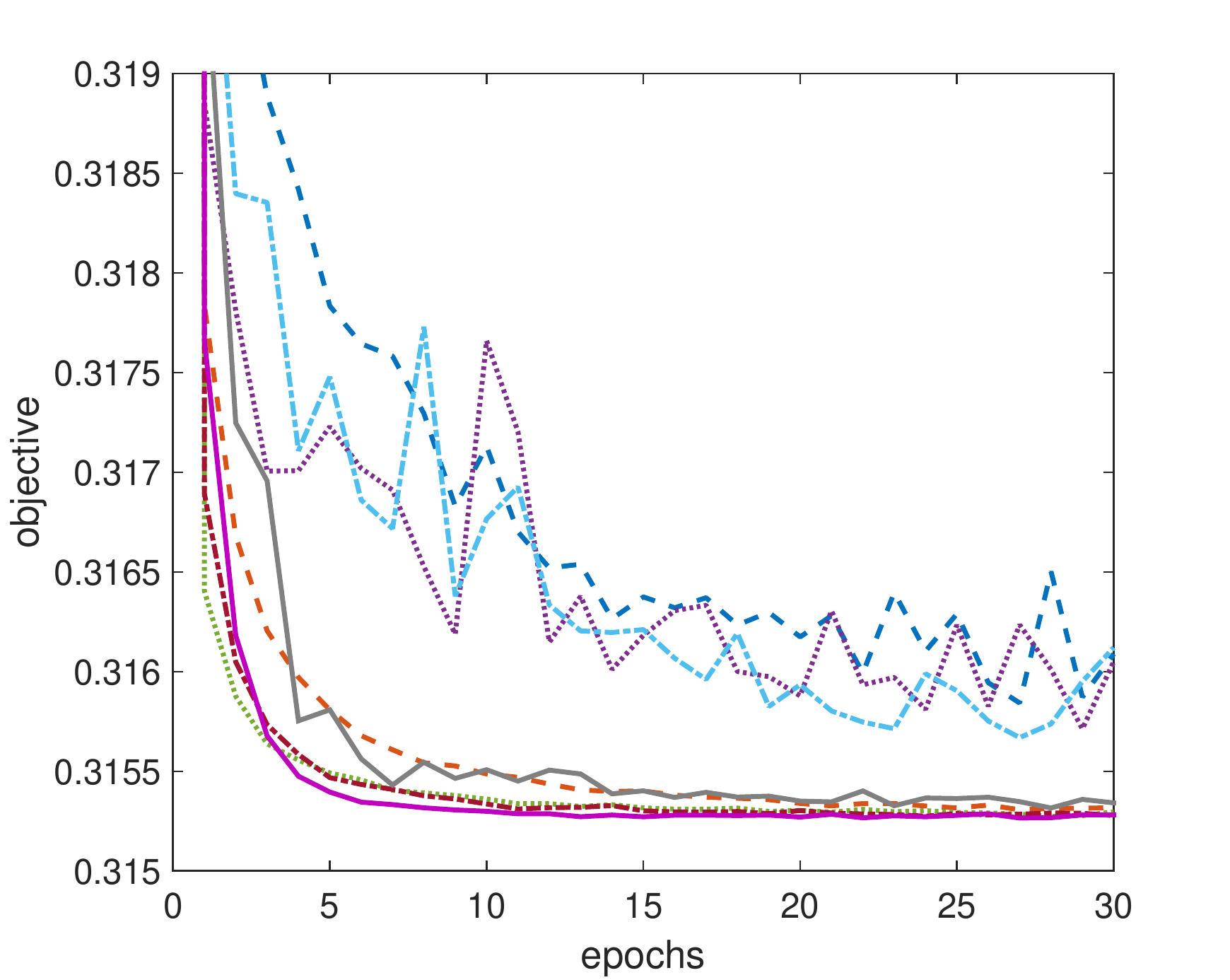}}
\subfigure[{\em avazu-app} ($\lambda = 10^{-8}$).]{\includegraphics[width=.65\columnwidth, height=.42\columnwidth]{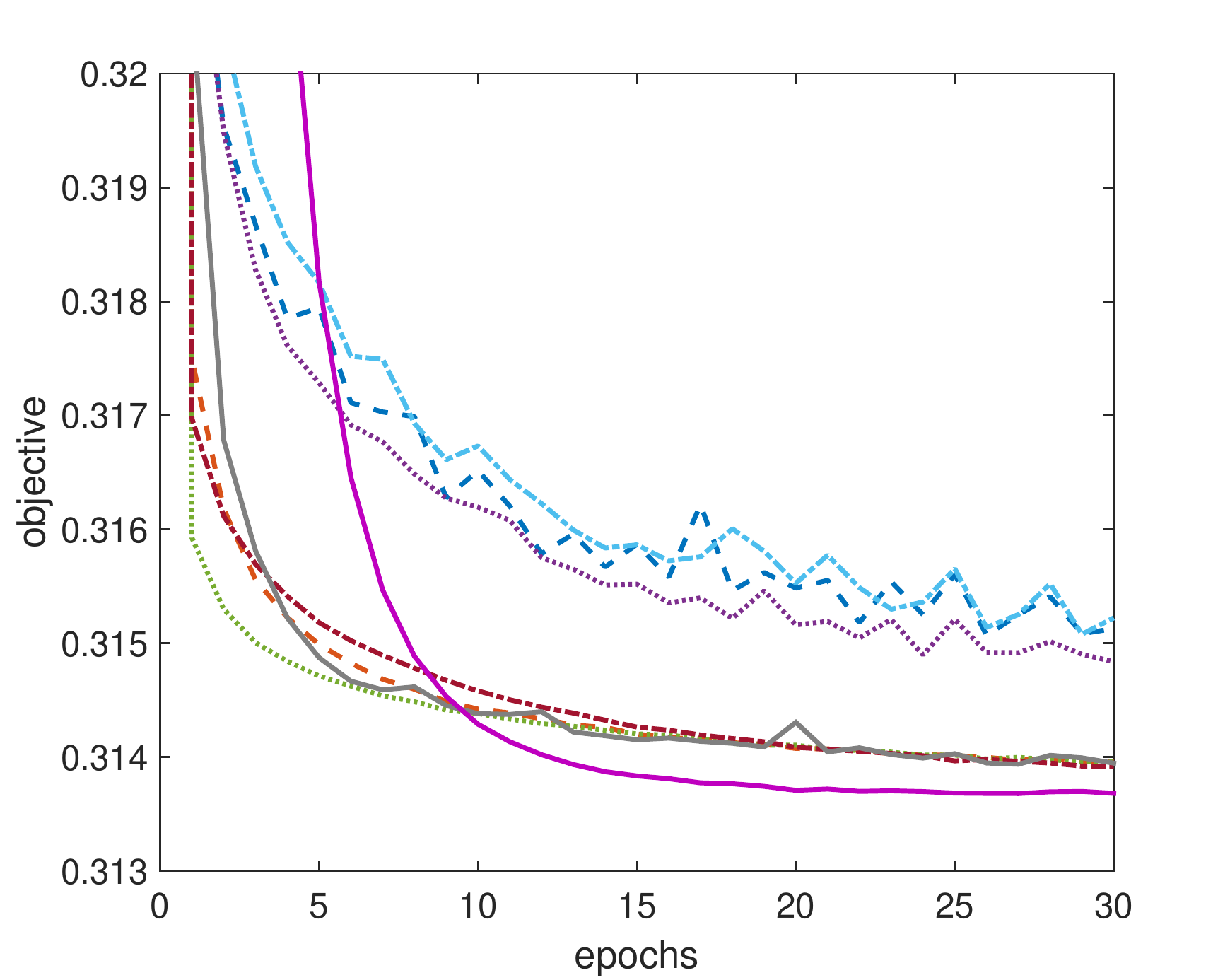}}
\vspace{-.1in}
\\
\subfigure[{\em kddb} ($\lambda = 10^{-6}$).]{\includegraphics[width=.65\columnwidth, height=.42\columnwidth]{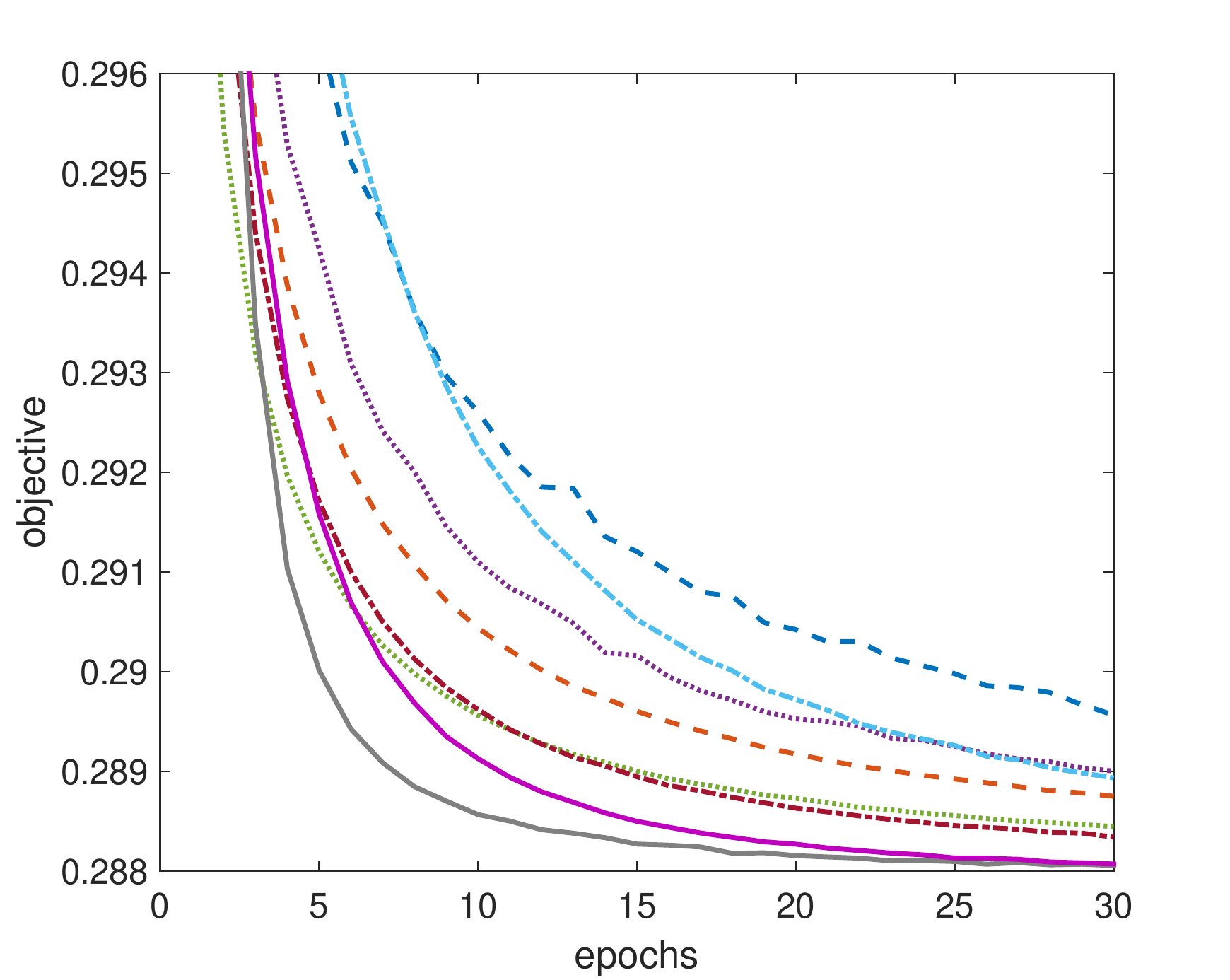}}
\subfigure[{\em kddb} ($\lambda = 10^{-7}$).]{\includegraphics[width=.65\columnwidth, height=.42\columnwidth]{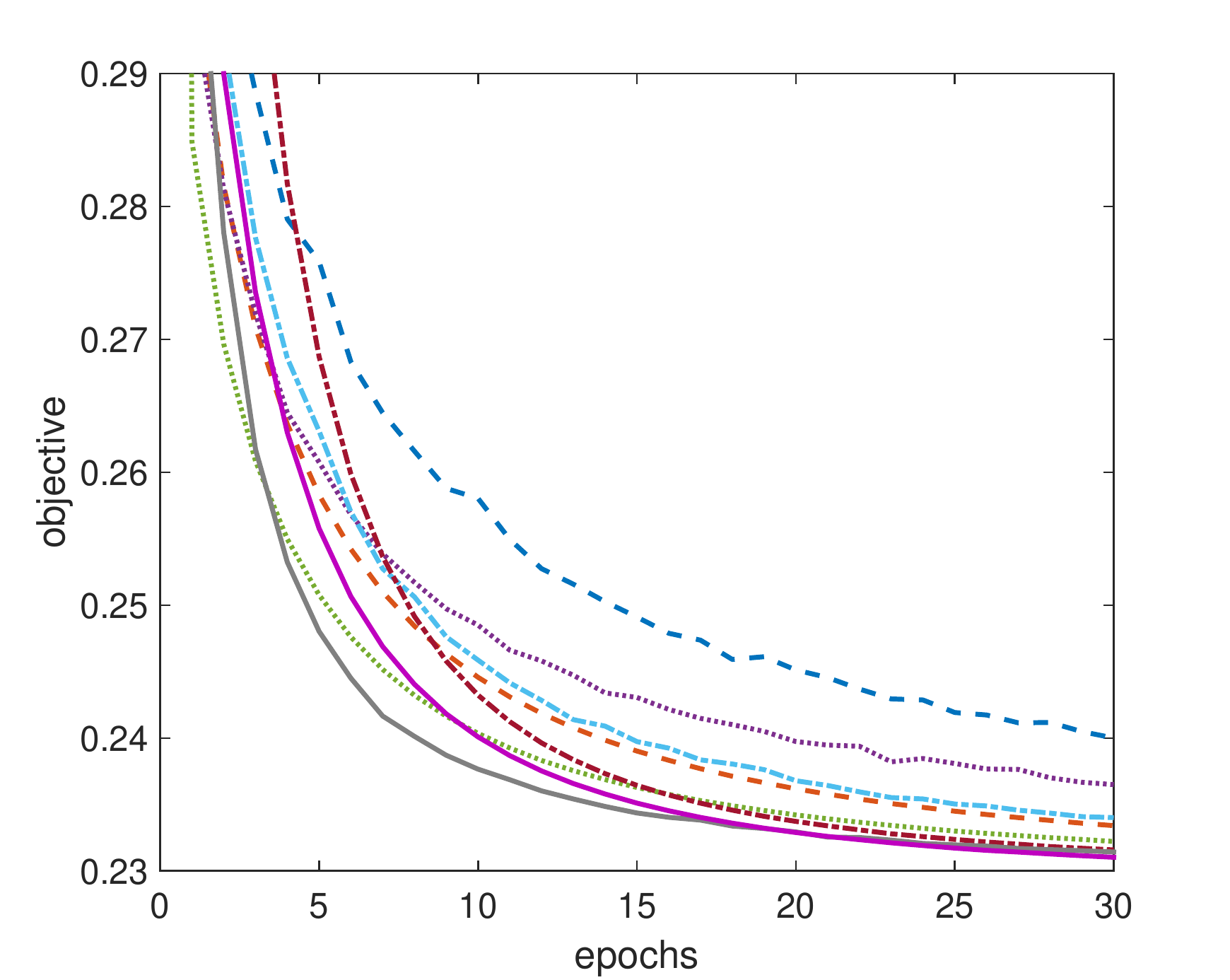}}
\subfigure[{\em kddb} ($\lambda = 10^{-8}$).]{\includegraphics[width=.65\columnwidth, height=.42\columnwidth]{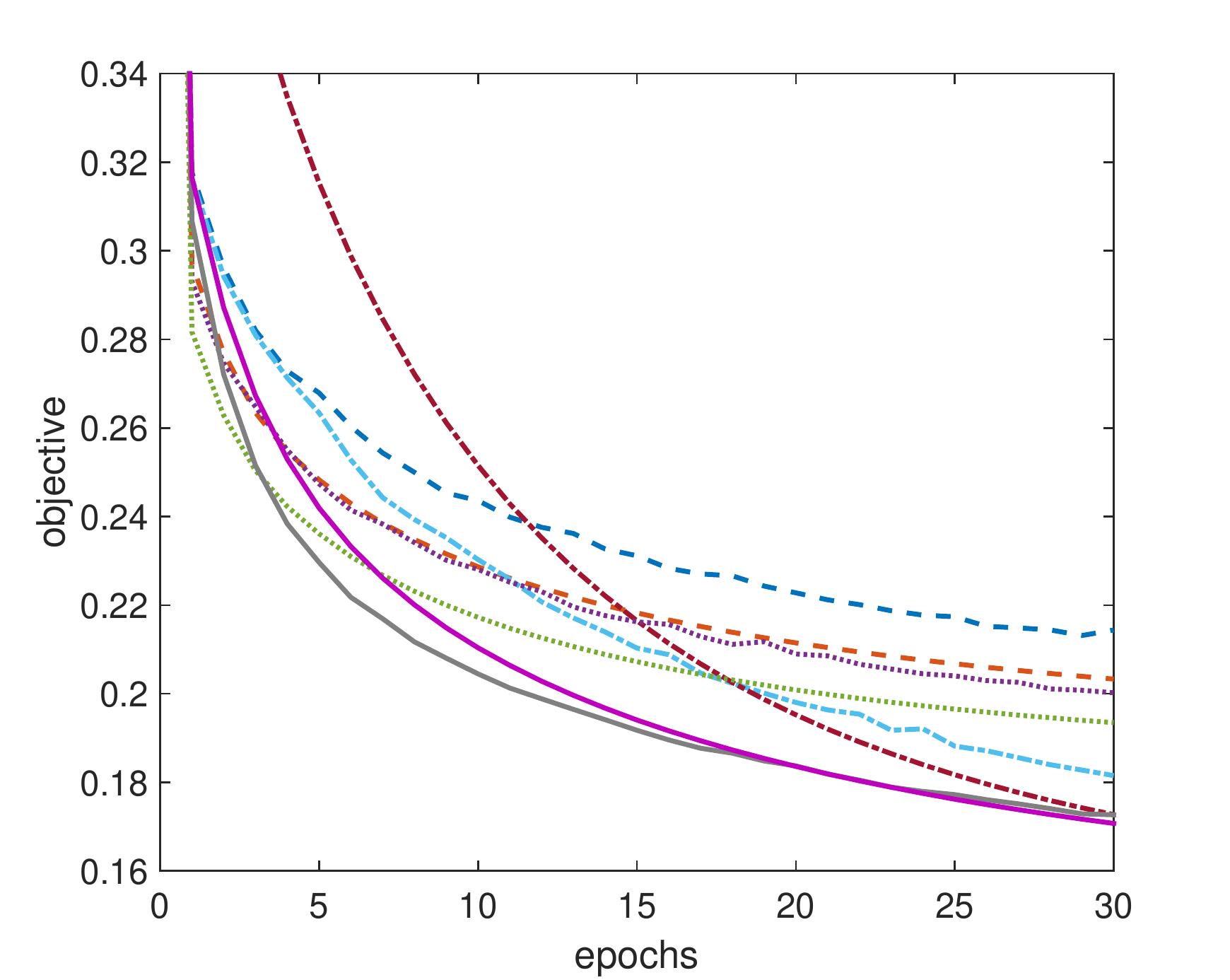}}
\setcounter{subfigure}{0}
\end{center}
\begin{center}
\includegraphics[width=6.5in]{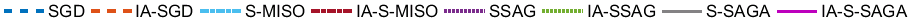}
\end{center}
\vspace{-.18in}
\caption{Convergence with the number of epochs
(both methods with and without iterate averaging are included). The experiment is performed on the same task as in Figure~\ref{classification_dropout_ni} but with more algorithms included.}
\label{classification_dropout}
\vspace{-.15in}
\end{figure*}

\begin{figure*}[t]
\begin{center}
\subfigure[{\em avazu-app} ($p = 0.1$).]{\includegraphics[width=.65\columnwidth, height=.42\columnwidth]{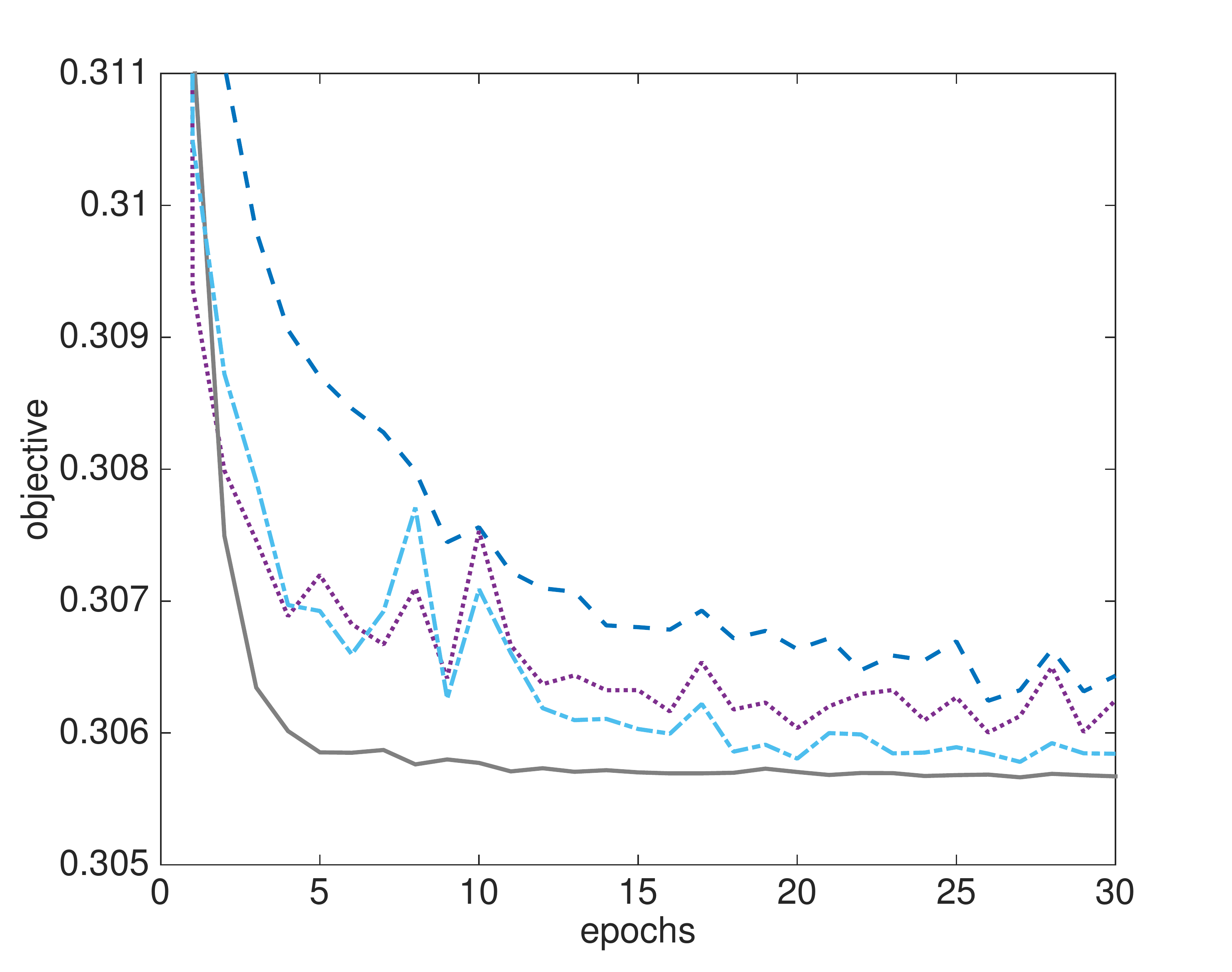}}
\subfigure[{\em avazu-app} ($p = 0.3$).]{\includegraphics[width=.65\columnwidth, height=.42\columnwidth]{avazu_train_7_ni.pdf}}
\subfigure[{\em avazu-app} ($p = 0.5$).]{\includegraphics[width=.65\columnwidth, height=.42\columnwidth]{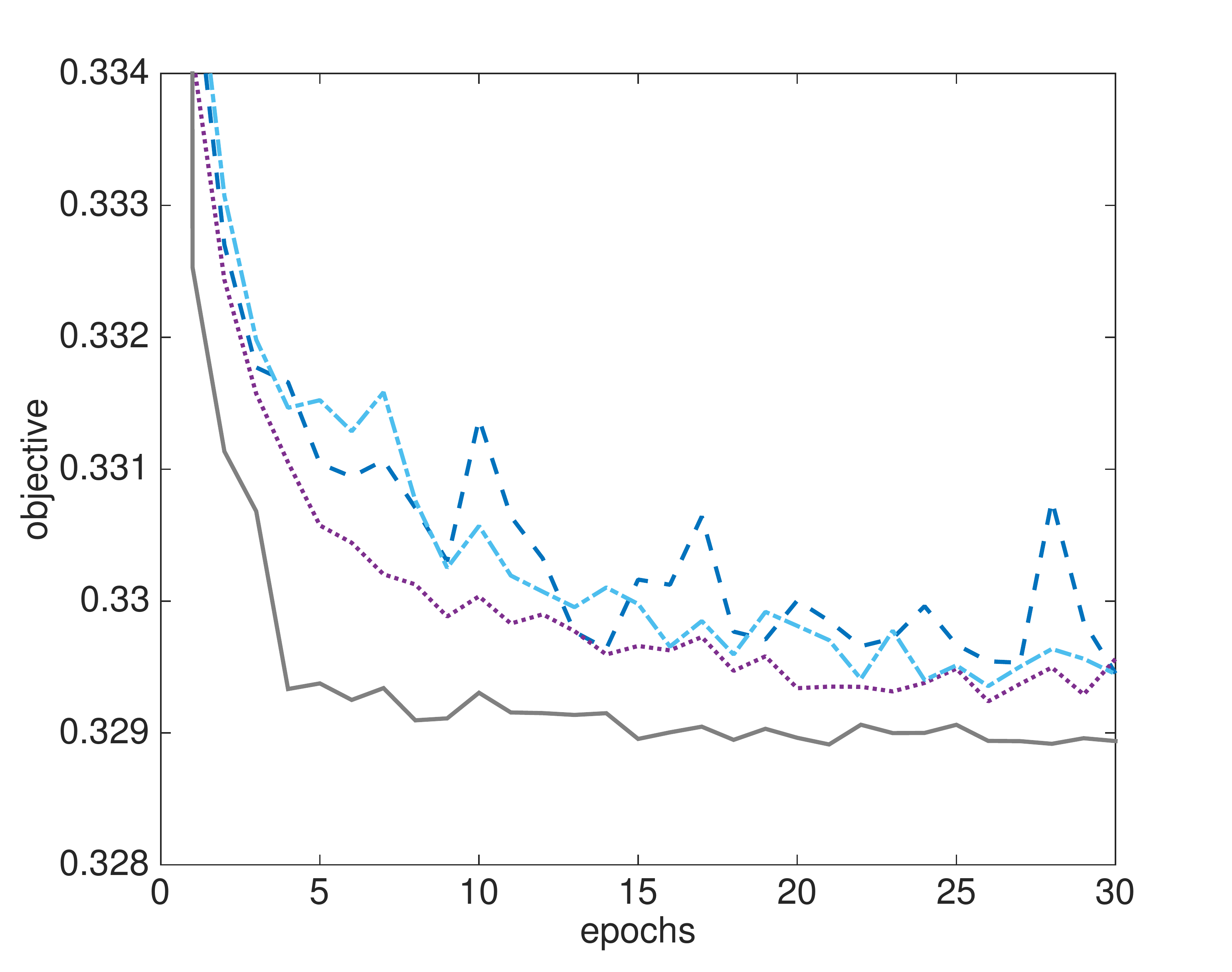}}
\setcounter{subfigure}{0}
\end{center}
\begin{center}
\includegraphics[width=3in]{legend_ni}
\end{center}
\vskip -.18in
\caption{Convergence with the number of epochs (logistic regression with dropout). The dropout probability is varied from $0.1$ to $0.5$.}
\label{classification_var_dropout_ni}
\vspace{-.15in}
\end{figure*}

\section{Experiments}

\begin{figure*}[h]
\begin{center}
\subfigure[{\em avazu-app} ($\lambda = 10^{-6}$).]{\includegraphics[width=.65\columnwidth, height=.42\columnwidth]{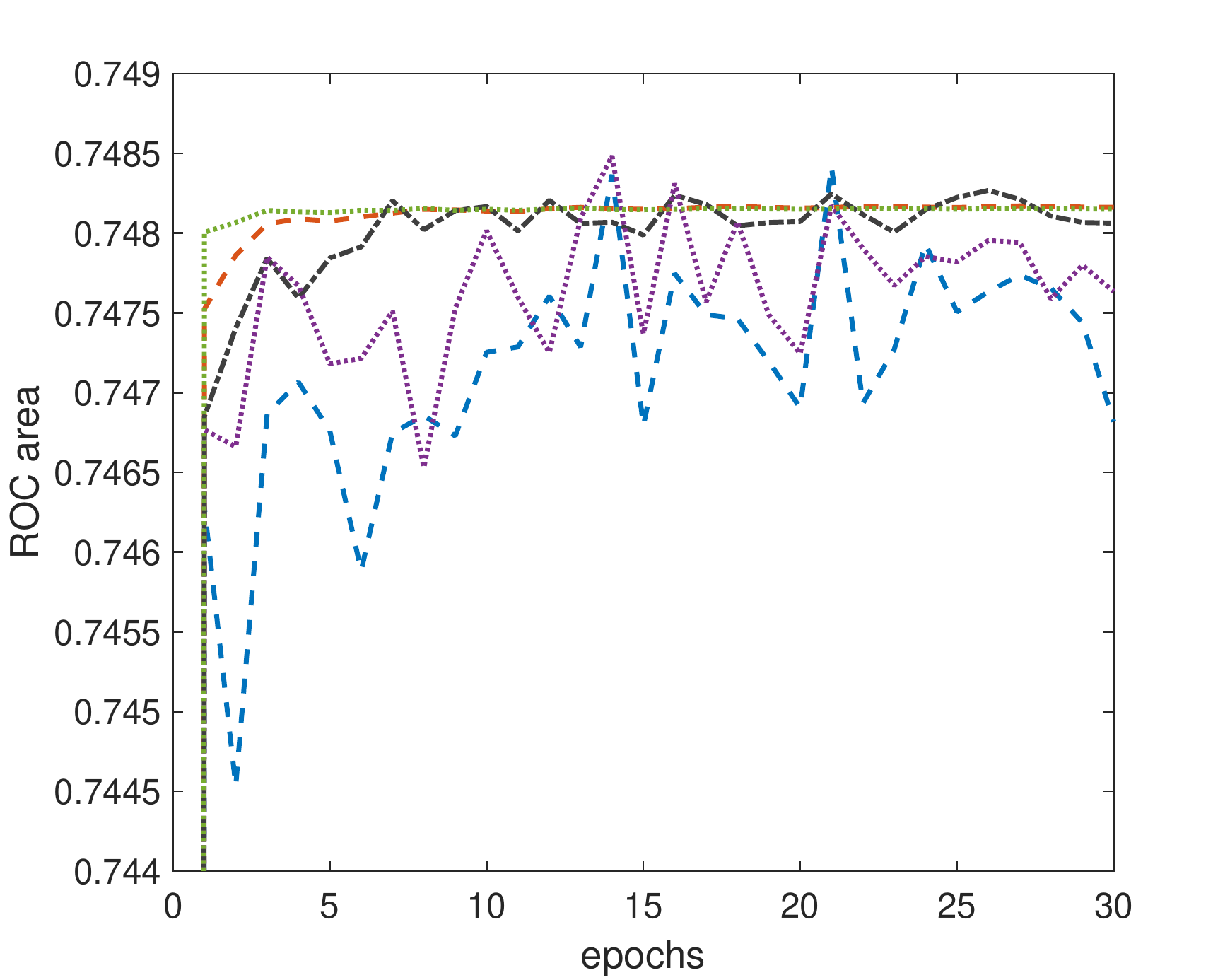}}
\subfigure[{\em avazu-app} ($\lambda = 10^{-7}$).]{\includegraphics[width=.65\columnwidth, height=.42\columnwidth]{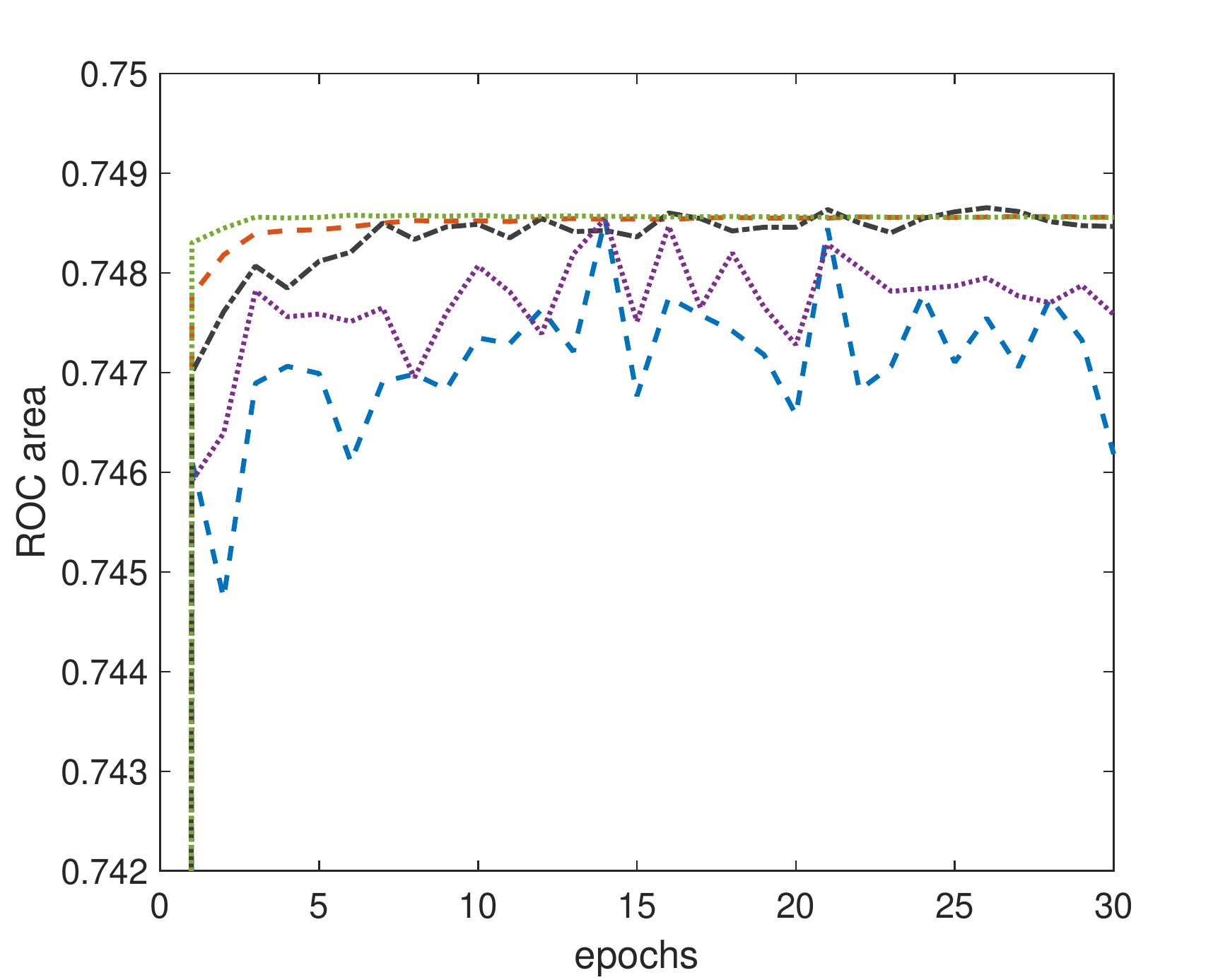}}
\subfigure[{\em avazu-app} ($\lambda = 10^{-8}$).]{\includegraphics[width=.65\columnwidth, height=.42\columnwidth]{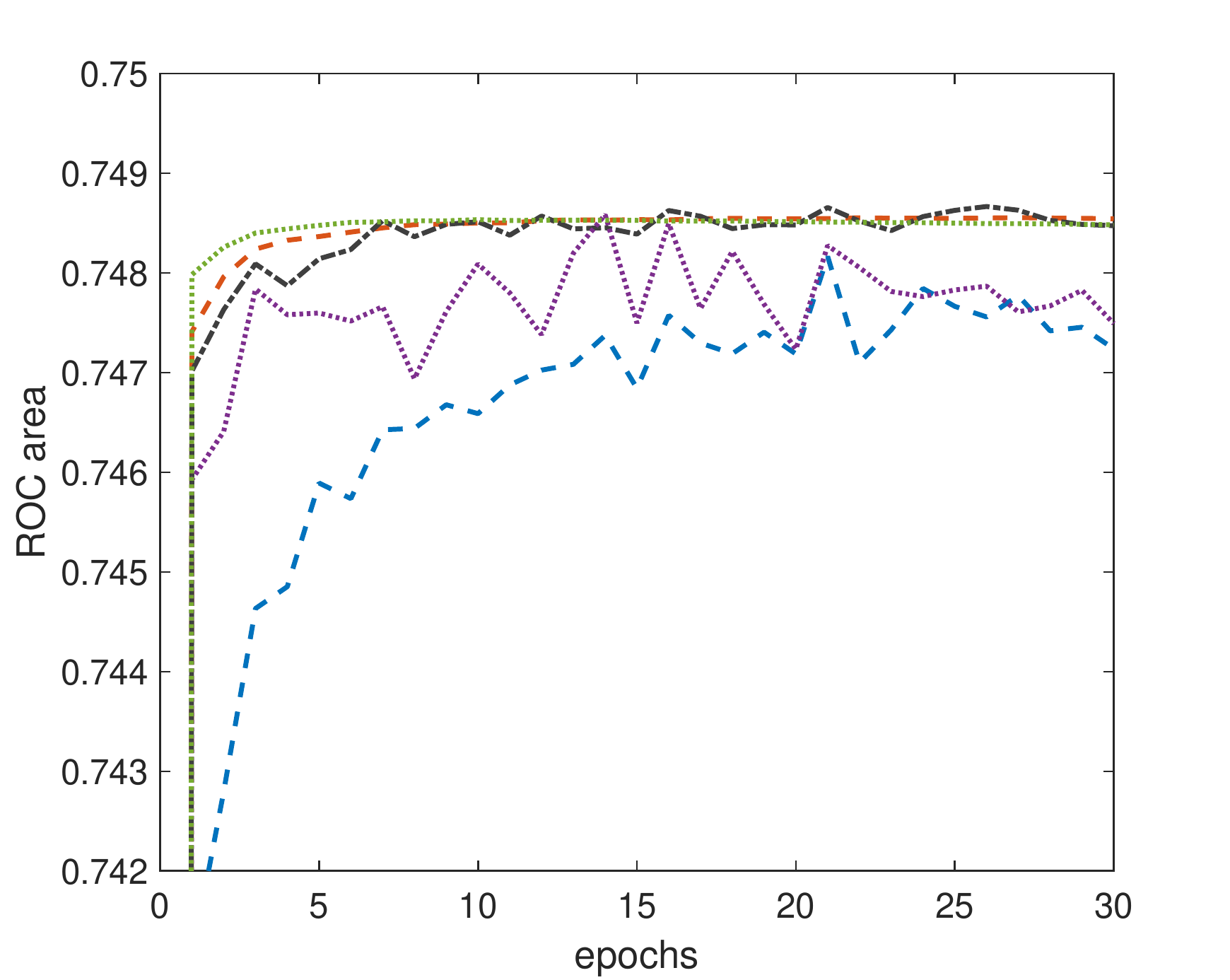}}
\vspace{-.12in}
\\
\subfigure[{\em kddb} ($\lambda = 10^{-5}$).]{\includegraphics[width=.65\columnwidth, height=.42\columnwidth]{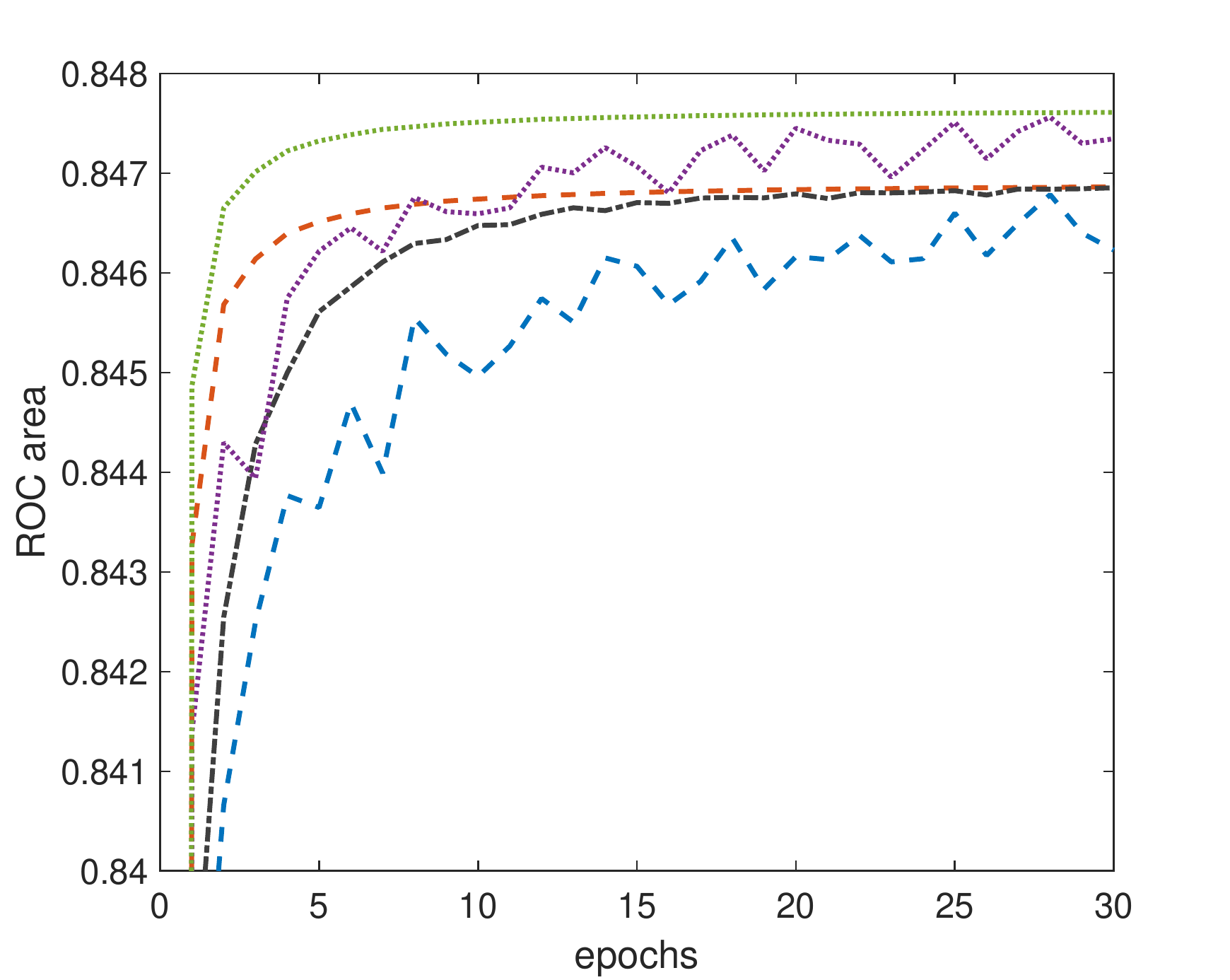}}
\subfigure[{\em kddb} ($\lambda = 10^{-6}$).]{\includegraphics[width=.65\columnwidth, height=.42\columnwidth]{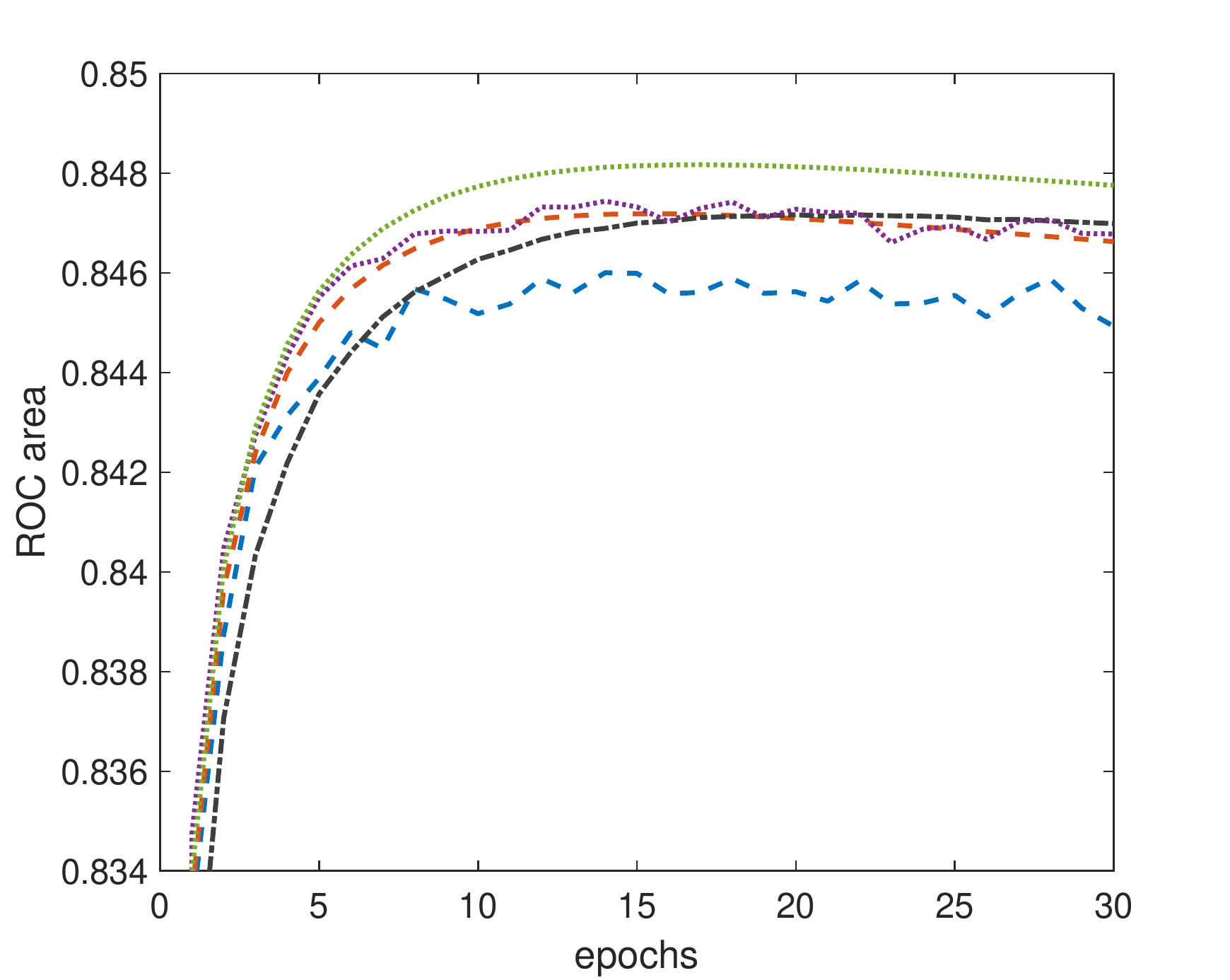}}
\subfigure[{\em kddb} ($\lambda = 10^{-7}$).]{\includegraphics[width=.65\columnwidth, height=.42\columnwidth]{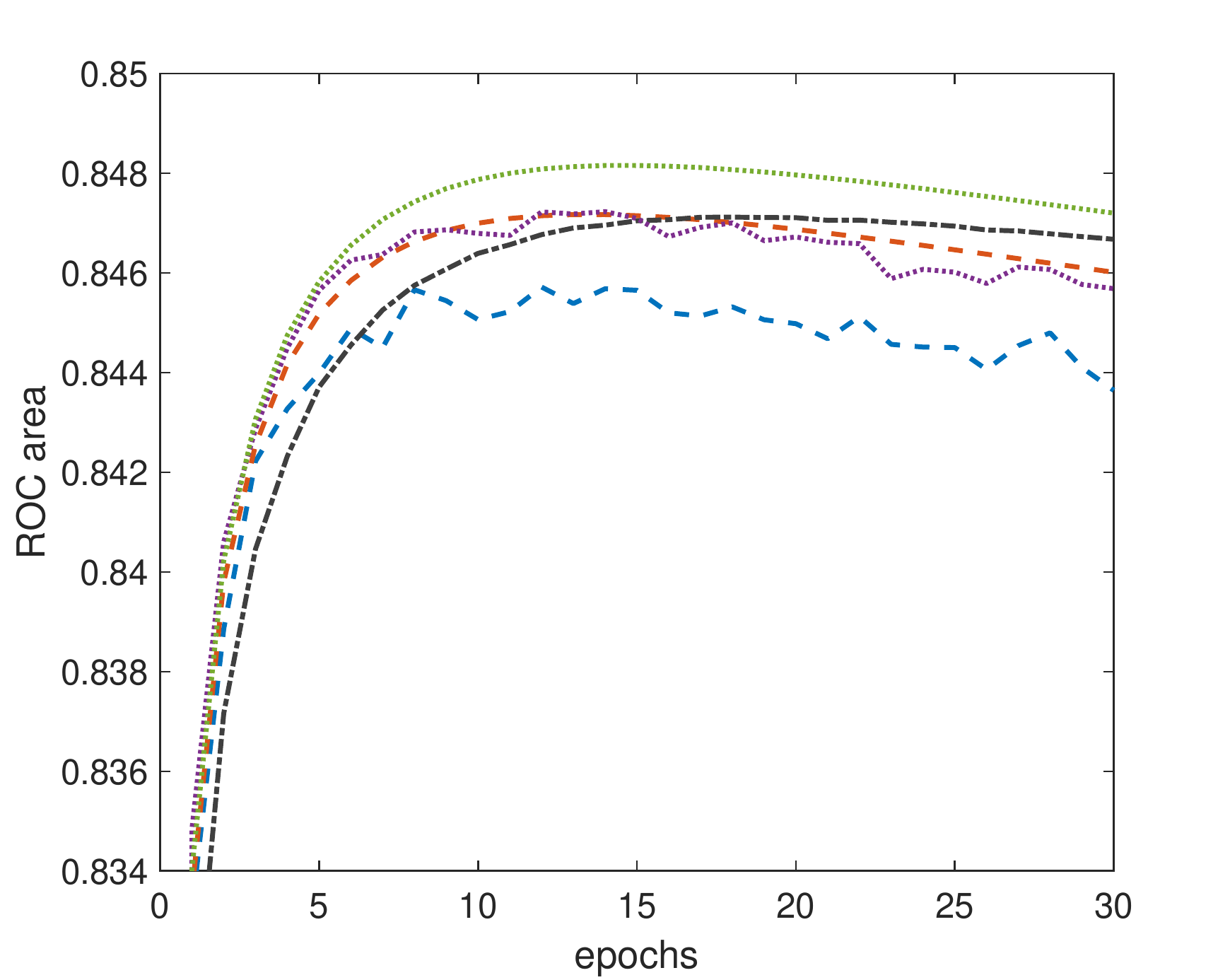}}
\setcounter{subfigure}{0}
\end{center}
\begin{center}
\includegraphics[width=3.5in]{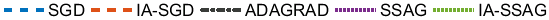}
\end{center}
\vskip -.15in
\caption{Convergence of AUC with the number of epochs. }
\label{roc_dropout}
\vspace{-.05in}
\end{figure*}

In this section, 
we perform experiments on 
logistic regression (Section~\ref{sec:dropout_sparse}) and AUC maximization (Section~\ref{sec:auc}).

\subsection{Logistic Regression with Dropout}
\label{sec:dropout_sparse}

Consider the $\ell_2$-regularized logistic regression model with dropout noise, 
with dropout probability $p=0.3$.
This can be formulated as the following optimization
problem:
\begin{equation} \label{eq:logi}
\min_\theta \frac{1}{n}\sum_{i=1}^n\CE_{\hat{\xi}_i}[\log(1 + \exp(-y_i\hat{z}_i^T\theta))] + \frac{\lambda}{2}\|\theta\|^2,
\end{equation} 
where $\hat{z}_i = \psi(z_i, \xi_i)$, $z_i$ 
is the feature vector of 
sample
$i$,
and $y_i$ the corresponding class label.
We vary $\lambda \in \{10^{-6}, 10^{-7}, 10^{-8}\}$. 
The smaller the $\lambda$, 
the higher the condition number. 
Experiments are performed on 
two 
high-dimensional  
data sets from the LIBSVM archive (Table~\ref{classification_data}).

\begin{table}[h] 
\vspace{-.1in}
\caption{Data sets used in the logistic regression experiment.}
\label{classification_data}
\begin{center}
\begin{tabular}{cccc}
\hline
& \#training & \#testing & dimensionality \\ \hline
{\em avazu-app}  &  12,642,186 & 1,953,951 & 1,000,000       \\
{\em kddb}  &  19,264,097 & 748,401 & 29,890,095        \\
\hline
\end{tabular}
\end{center}
\vspace{-.1in}
\end{table} 

\subsubsection{Comparison with SGD and S-MISO}

The 
proposed SSAG and
S-SAGA
are compared with
SGD and
S-MISO.
From Proposition~\ref{lemma:opt_beta_conv}, we use a slightly larger $\beta_t = t^{-0.75}$ for better non-asymptotic performance.
As mentioned in the theorems, the stepsize schedule is $\eta_t=c/(\gamma + t)$. 
We fix $c=2/\lambda$ for SGD, SSAG, S-SAGA, and $c=2n$ for S-MISO as suggested in \cite{bietti2016stochastic}. 
We then select $\gamma$ from a number of possible values (e.g., powers of tens and five times powers of tens) by monitoring the training objective.
To reduce statistical variability, results are averaged over five repetitions. 


As all methods under comparison have the same
iteration complexities,
Figure~\ref{classification_dropout_ni} shows
convergence of the training objective with 
the number of epochs.
The expectation
in (\ref{eq:logi}) is estimated from $5$ perturbed samples.
As can be seen, S-SAGA significantly outperforms all the others. In particular, it
reaches a much lower objective value 
when the condition number is large
($\lambda = 10^{-8}$).
SSAG and S-MISO have similar convergence behavior 
and converge faster than SGD. However, S-MISO requires much more memory 
than SSAG.
A comparison of the
additional memory 
(relative to SGD) 
used by each method 
is shown in Table~\ref{classification_memory}.

\begin{table}[h]
 \vskip -.1in 
\caption{Additional memory (relative to SGD) required by the various algorithms in the logistic regression experiment.}
\label{classification_memory}
\begin{center}
\begin{tabular}{cccc}
\hline
&  S-MISO & SSAG & S-SAGA \\ \hline
{\em avazu-app}  &  3.1GB &  7.6MB &   104.1MB    \\
{\em kddb}  & 8.9GB  & 147 MB &     375MB    \\
\hline
\end{tabular}
\end{center}
\vskip -.1in
\end{table} 

To see how $a_t$ differs from $a_t^*$ 
in 
(\ref{em_opt_a}), we perform an experiment using a subset of {\em covertype} data from the LIBSVM archive. 
The expectations in $a_t^*$ are again approximated by randomly sampling 5 perturbations for each sample. 
Empirically, $\max_{t\geq 2} |a_t - a_t^*|/|a_t^*|$ is of the order $0.01$, indicating that
$a_t$ is a reasonable estimate even in the early iterations.

\subsubsection{Use of Iterate Averaging}

Figure~\ref{classification_dropout}
adds 
the convergence results 
for iterate averaging 
to Figure~\ref{classification_dropout_ni}
(``IA" 
is prepended 
to the names of algorithms using
iterate averaging).
As can be seen, iterate averaging leads to significant improvements for SGD, S-MISO and SSAG,
but
less prominent improvement 
for S-SAGA.
This agrees with the discussions in Section~\ref{sec:ia}. 
Moreover, 
when 
the condition number is
high,
SSAG has similar convergence as IA-SGD on {\em kddb}.
This demonstrates that SSAG is more robust to 
large 
condition number.  


Overall, when memory is not an issue, S-SAGA is preferred for problems with small or medium
condition numbers, while IA-S-SAGA can be better
for problems with large condition numbers. 
If 
memory
is limited, IA-SSAG is recommended.

\subsubsection{Varying the Dropout Probability}

In this section, we study how the strength of the dropout noise affects
convergence.
We 
use the {\em avazu-app} data set, and
fix $\lambda = 10^{-7}$. The dropout probability $p$ is varied in $\{0.1, 0.3,
0.5\}$. Note that a larger dropout probability leads to larger noise variance. 
Figure~\ref{classification_var_dropout_ni} shows that S-SAGA is very robust to
different noise levels, while S-MISO performs much worse when the dropout
probability increases. 
This demonstrates the theoretical result in
Theorem~\ref{em_conv_ssaga}
that S-SAGA has a smaller variance
constant,
while SGD and SSAG are not sensitive to 
$p$.

\subsection{AUC Maximization with Dropout}
\label{sec:auc}

In this section, we consider maximization of the AUC (i.e., area under the ROC curve).
This is equivalent to 
ranking the positive samples 
higher 
than the negative samples 
\cite{sculley2009large}. 
It can be formulated as  minimizing the following objective with
the squared hinge loss:
\begin{eqnarray*}
\frac{1}{n^+n^-} \!\!\!\!\!\sum_{y_i = 1, y_j = 0}
\!\!\!\!\!\! \CE_{\xi_i, \xi_j}[\max(0, 1 - (\hat{z}_i - \hat{z}_j)^T\theta)^2] + \frac{\lambda}{2}\|\theta\|^2,
\end{eqnarray*}
where $n^+, n^-$ are the numbers of samples belonging to the positive and negative class, respectively.  
We again use the data sets in Table~\ref{classification_data},
and inject dropout noise 
with dropout probability $p=0.3$. 


Even without noise perturbation, AUC maximization is infeasible for existing variance reduction methods.
Methods such as SAG and SAGA  need $O(n^+n^-)$ space.
SVRG trades space with time, and takes $O(n^+n^-)$ time. 
With dropout noise injected, S-MISO requires even more 
space, 
namely, $O(n^+n^-d)$.
S-SAGA requires $O(n^+n^-)$ space, and so is also impractical.
Thus, in the following, we only compare SGD, SSAG and their variants with iterate averaging. 
As a further baseline, we also compare with
ADAGRAD
\cite{duchi2011adaptive},
which  performs SGD with an adaptive learning rate.

Figure~\ref{roc_dropout} shows the 
results.
IA-SSAG is always the fastest, and
has the highest AUC on {\em kddb}.
ADAGRAD and IA-SGD have comparable AUC with IA-SSAG on {\em avazu-app}, but not on {\em kddb}.
ADAGRAD is faster than SGD and SSAG on {\em avazu-app}, but slower than SSAG on  {\em kddb}.
On {\em kddb}, SSAG has comparable performance with IA-SGD, and is better when $\lambda = 10^{-5}$.



\section{Conclusion}

In this paper, we proposed two SGD-like algorithms for finite sums with infinite
data when learning with the linear model. 
The key is to exploit the linear structure in the construction of
control variates. Convergence results on strongly convex problems are provided. 
The proposed methods require small memory cost. Experimental results
demonstrate that the proposed algorithms outperform the state-of-the-art
on large data sets.

\bibliography{as}
\bibliographystyle{icml2018}


\newpage
\onecolumn
\appendix

\section{Extension to Composite Objectives}

In this section, we study extending SSAG to the composite objective:
\begin{eqnarray} \label{eq:em_com_problem}
\min_{\theta} P(\theta)  \equiv F(\theta)+ r(\theta),
\end{eqnarray}
where $r(\theta)$ is a (possibly nonsmooth) convex function. 
We assume that 
the proximal operator of $r$,
$\text{prox}_{\eta r}(q)  \equiv \arg\min_{\theta}\left\{\eta r(\theta) + \frac{1}{2}\|\theta - q\|^2\right\}$,
can be easily computed.
In Algorithms~\ref{alg:ssag1} and \ref{alg:ssag2},
step 8 can then be simply replaced
by 
\[\theta_t  \leftarrow \text{prox}_{\eta_t r}(\theta_{t-1} - \eta_tv_t).\] 
The following Theorem shows that this proximal variant of SSAG still achieves a convergence rate of $O(1/t)$. 

\begin{theorem} \label{em_p_conv}
Assume that $\eta_t = c/(\gamma + t)$ for some $c > 1/\mu$ and $\gamma > 0$ such that $\eta_1 \leq 1/L$.
For the $\{\theta_t\}$ sequence generated from 
proximal SSAG, we have
\begin{eqnarray} \label{eq:em_conv_pvrsgd}
\CE[\|\theta_t - \theta_*\|^2] \leq \frac{\nu_5}{\gamma + t + 1},
\end{eqnarray}
where
$\nu_5 \equiv \max\{\frac{2c^2\sigma_a^2(\gamma+2)}{(\gamma+1)(2c\mu-1) + c\mu}, (\gamma +
1)\|\theta_0 -\theta_*\|^2\}$.
\end{theorem}


As in Section~\ref{sec:ia},
iterate averaging 
can be used
to reduce the
$\kappa^2$
term to $\kappa$.

\begin{theorem} \label{em_ap_conv}
Assume that
$\eta_t = 2/(\mu(\gamma + t))$ and $\gamma > 0$ such that $\eta_1 \leq 1/L$.
For the $\{\theta_t\}$ sequence generated from 
proximal SSAG, we have
\begin{eqnarray*} \label{eq:em_a_conv_pvrsgd}
\lefteqn{\CE[P(\bar{\theta}_T) - P(\theta_*)]}\\
& \leq &  \frac{\mu\gamma(\gamma - 1)}{2T(2\gamma + T - 1)}\|\theta_0 - \theta_*\|^2  + \frac{4\sigma_a^2}{\mu(2\gamma + T - 1)}.
\end{eqnarray*}
\end{theorem}
Similar results also hold for
proximal S-SAGA.

\section{proof}


\subsection{Proof of Proposition~\ref{prop:asym_unbiased_est}}
\begin{proof} 
\begin{eqnarray*}
\CE[z_t] & = & \nabla F(\theta_{t-1})  - a_t\CE[\hat{x}_{i_t}] + a_t\CE[\tx_t] \\
& = & \nabla F(\theta_{t-1})  - a_t\CE[\hat{x}_{i_t}] + a_t\left(1 - \frac{1}{t}\right)\tx_{t-1} + a_t\frac{1}{t}\CE[\hat{x}_{i_t}] \\
& =&
\nabla F(\theta_{t-1})+ a_t\left(1 - \frac{1}{t}\right)(\tx_{t-1} -\CE[\hat{x}_{i_t}]).
\end{eqnarray*}
\end{proof} 


\subsection{Proof of Proposition~\ref{prop:em_var}}
\begin{proof} 
The estimator $v_t$ is unbiased simply due to $\CE[\hat{x}_{i_t}] = \tx$. 
\begin{eqnarray*} 
\lefteqn{\CE[\|v_t - \nabla F(\theta_{t-1})\|^2]}\nonumber \\
& = & \CE[\|(\phi_{i_t}'(\hat{x}_{i_t} ^T\theta_{t-1}) - a_t)\hat{x}_{i_t}  + a_t\tx + \nabla g(\theta_{t-1})-  \nabla F(\theta_{t-1})\|^2] \\
& = & \CE[\|(\phi_{i_t}'(\hat{x}_{i_t} ^T\theta_{t-1}) - a_t)\hat{x}_{i_t}  - \CE[(\phi'(\hat{x}^T\theta_{t-1}) - a_t)\hat{x}]\|^2] \\
 & \leq & \CE[(\phi_{i_t}'(\hat{x}_{i_t}^T\theta_{t-1}) - a_t)^2\|\hat{x}_{i_t}\|^2],
\end{eqnarray*}
where the inequality follows from $\CE[\|x - \CE[x]\|^2] = \CE[\|x\|^2] - \|\CE[x]\|^2$.
\end{proof} 


\subsection{Proof of Proposition~\ref{lemma:opt_beta}}
\begin{proof} 
This lemma is a consequence of (\cite{ruszczynski1980feasible}, Lemma~1). To see this, we just need to verify 
that all the conditions are satisfied by the problem at hand. 
Condition (a) of (\cite{ruszczynski1980feasible}, Lemma~1) is obviously satisfied by treating set $Z$ in (\cite{ruszczynski1980feasible}, Lemma~1) as entire real space $\R$. 
Condition (b) is exactly the finite fourth moment assumption.
Conditions (c)-(d) hold due to the stepsize policy of $\beta_t$. 
Condition (e) follows from the Lipschitz property
of loss $\phi$ from Assumption~\ref{assumption:lipschitz}, bounded estimator $v_t$, and stepsize rules of $\eta_t$ and $\beta_t$.
\end{proof}


\subsection{Proof of Proposition~\ref{lemma:opt_beta_conv}}
\begin{proof} 
Conditioned on the history, taking expectation over $i_{t-1}$ and $\xi_{t-1}$, for any $t \geq 2$, we have
\begin{eqnarray*}
\lefteqn{\CE[s_t^2(a_t - a_t^*)^2]} \\
& = &  \CE[s_t^2\left(\tilde{a}_t/s_t - a_t^*\right)^2] \\
& = &  \CE\left[s_t^2\left(\frac{(1-\beta_{t-1})\tilde{a}_{t-1} + \beta_{t-1}\phi'_{i_{t-1}}(\hat{x}_{i_{t-1}}^T\theta_{t-2})\|\hat{x}_{i_{t-1}}\|^2}{(1-\beta_{t-1})s_{t-1} + \beta_{t-1}\|\hat{x}_{i_{t-1}}\|^2} - a_t^*\right)^2\right] \\
& = &  \CE\left[s_t^2\left(\frac{(1-\beta_{t-1})(\tilde{a}_{t-1} - s_{t-1}a_t^*) + \beta_{t-1}\|\hat{x}_{i_{t-1}}\|^2(\phi'_{i_{t-1}}(\hat{x}_{i_{t-1}}^T\theta_{t-2}) - a_t^*)}{(1-\beta_{t-1})s_{t-1} + \beta_{t-1}\|\hat{x}_{i_{t-1}}\|^2} \right)^2\right] \\
& = &  \CE\left[\left((1-\beta_{t-1})(s_{t-1}a_{t-1} - s_{t-1}a_{t-1}^* + s_{t-1}a_{t-1}^* - s_{t-1}a_t^*) + \beta_{t-1}\|\hat{x}_{i_{t-1}}\|^2(\phi'_{i_{t-1}}(\hat{x}_{i_{t-1}}^T\theta_{t-2}) - a_t^*) \right)^2\right] \\
& = &  \CE\left[\left((1-\beta_{t-1})s_{t-1}(a_{t-1} - a_{t-1}^*) + s_{t}(a_{t-1}^* - a_t^*) + \beta_{t-1}\|\hat{x}_{i_{t-1}}\|^2(\phi'_{i_{t-1}}(\hat{x}_{i_{t-1}}^T\theta_{t-2}) - a_{t-1}^*)\right)^2\right] \\
& = &  \CE\left[(1-\beta_{t-1})^2s_{t-1}^2(a_{t-1} - a_{t-1}^*)^2\right] + \CE\left[s_t^2(a_{t-1}^* - a_{t}^*)^2\right] 
 + \CE\left[\beta_{t-1}^2\|\hat{x}_{i_{t-1}}\|^4(\phi'_{i_{t-1}}(\hat{x}_{i_{t-1}}^T\theta_{t-2}) - a_{t-1}^*)^2\right] \\ 
&& + \CE\left[2(1 - \beta_{t-1})s_{t-1}s_t(a_{t-1} - a_{t-1}^*)(a_{t-1}^* - a_t^*)\right],
\end{eqnarray*}
where the last equality holds as $\CE[\|\hat{x}_{i_{t-1}}\|^2(\phi'_{i_{t-1}}(\hat{x}_{i_{t-1}}^T\theta_{t-2}) - a_{t-1}^*)] = 0$. Then, we further bound the inner product as
\begin{eqnarray*}
2s_{t-1}s_t(a_{t-1} - a_{t-1}^*)(a_{t-1}^* - a_t^*) \leq \alpha_ts_{t-1}^2(a_{t-1} - a_{t-1}^*)^2 + \frac{s_t^2}{\alpha_t}(a_{t-1}^* - a_t^*)^2
\end{eqnarray*}
for some $\alpha_t > 0$. Combining this, we obtain
\begin{eqnarray*}
\lefteqn{\CE[s_t^2(a_t - a_t^*)^2]} \\
& \leq &  \CE\left[(1-\beta_{t-1})((1-\beta_{t-1}) + \alpha_t)\right]s_{t-1}^2(a_{t-1} - a_{t-1}^*)^2\\
&& + \CE\left[(1 + \frac{1}{\alpha_t})s_t^2(a_{t-1}^* - a_{t}^*)^2\right] 
 + \CE\left[\beta_{t-1}^2\|\hat{x}_{i_{t-1}}\|^4(\phi'_{i_{t-1}}(\hat{x}_{i_{t-1}}^T\theta_{t-2}) - a_{t-1}^*)^2\right].
\end{eqnarray*}
The Lipschitz continuity of the gradients implies that
\begin{eqnarray*}
(a_{t-1}^* - a_{t}^*)^2 & = &  \frac{1}{\CE[\|\hat{x}\|^2]^2}(\CE[\phi'(\hat{x}^T\theta_{t-2})\|\hat{x}\|^2] - \CE[\phi'(\hat{x}^T\theta_{t-1})\|\hat{x}\|^2])^2 \\
& \leq & \frac{1}{\CE[\|\hat{x}\|^2]^2}\CE[(\phi'(\hat{x}^T\theta_{t-2})\|\hat{x}\|^2 - \phi'(\hat{x}^T\theta_{t-1})\|\hat{x}\|^2)^2] \\
& \leq & \frac{L^2}{\CE[\|\hat{x}\|^2]}\|\theta_{t-2} - \theta_{t-1}\|^2 \\
& \leq & \frac{\eta_{t-1}^2 L^2}{\CE[\|\hat{x}\|^2]}\|v_{t-1}\|^2 \\
& \leq & \frac{\eta_{t-1}^2 L^2}{\CE[\|\hat{x}\|^2]}G^2,
\end{eqnarray*}
where we assume that $\|v_t\| \leq G$ for all $t$. Then, substituting it and rearranging, we obtain
\begin{eqnarray*}
\lefteqn{\CE[s_t^2(a_t - a_t^*)^2]} \\
& \leq &  (1-\beta_{t-1})(1-\beta_{t-1} + \alpha_t)s_{t-1}^2(a_{t-1} - a_{t-1}^*)^2\\
&& + (1 + \frac{1}{\alpha_t})\eta_{t-1}^2 L^2G^2\frac{\CE[s_{t}^2]}{\CE[\|\hat{x}\|^2]} 
 + \CE\left[\beta_{t-1}^2\|\hat{x}_{i_{t-1}}\|^4(\phi'_{i_{t-1}}(\hat{x}_{i_{t-1}}^T\theta_{t-2}) - a_{t-1}^*)^2\right].
\end{eqnarray*}
Let $\alpha_t  = \beta_{t-1}$, we have
\begin{eqnarray*}
\lefteqn{\CE[s_t^2(a_t - a_t^*)^2]} \\
& \leq &  (1-\beta_{t-1})s_{t-1}^2(a_{t-1} - a_{t-1}^*)^2
+ \eta_{t-1}^2 L^2G^2\frac{\CE[s_{t}^2]}{\CE[\|\hat{x}\|^2]} + \frac{\eta_{t-1}^2}{\beta_{t-1}} L^2G^2\frac{\CE[s_{t}^2]}{\CE[\|\hat{x}\|^2]} 
\\
&& + \CE\left[\beta_{t-1}^2\|\hat{x}_{i_{t-1}}\|^4(\phi'_{i_{t-1}}(\hat{x}_{i_{t-1}}^T\theta_{t-2}) - a_{t-1}^*)^2\right].
\end{eqnarray*}
Then, we bound the last term as 
\begin{eqnarray*}
\lefteqn{\CE\left[\|\hat{x}_{i_{t-1}}\|^4(\phi'_{i_{t-1}}(\hat{x}_{i_{t-1}}^T\theta_{t-2}) - a_{t-1}^*)^2\right]}\\
& =& \CE\left[(\phi'_{i_{t-1}}(\hat{x}_{i_{t-1}}^T\theta_{t-2})\|\hat{x}_{i_{t-1}}\|^2 - \CE[\phi'(\hat{x}^T\theta_{t-2})\|\hat{x}\|^2] + \CE[\phi'(\hat{x}^T\theta_{t-2})\|\hat{x}\|^2]  - a_{t-1}^*\|\hat{x}_{i_{t-1}}\|^2)^2\right] \\
& \leq & 2\CE\left[(\phi'_{i_{t-1}}(\hat{x}_{i_{t-1}}^T\theta_{t-2})\|\hat{x}_{i_{t-1}}\|^2 - \CE[\phi'(\hat{x}^T\theta_{t-2})\|\hat{x}\|^2])^2\right]
+ 2\CE\left[(\CE[\phi'(\hat{x}^T\theta_{t-2})\|\hat{x}\|^2]  - a_{t-1}^*\|\hat{x}_{i_{t-1}}\|^2)^2\right] \\
& =& 2\CE\left[(\phi'(\hat{x}^T\theta_{t-2})\|\hat{x}\|^2 - \CE[\phi'(\hat{x}^T\theta_{t-2})\|\hat{x}\|^2])^2\right]
+ 2(a_{t-1}^*)^2\CE\left[(\|\hat{x}\|^2 - \CE[\|\hat{x}\|^2])^2\right] \\
& \leq & 2\sigma_{\phi}^2 + 2\bar{a}^2\sigma_{b}^2,
\end{eqnarray*}
where we assume that $\CE\left[(\phi'(\hat{x}^T\theta_{t-2})\|\hat{x}\|^2 - \CE[\phi'(\hat{x}^T\theta_{t-2})\|\hat{x}\|^2])^2\right] \leq \sigma_{\phi}^2$,  
$\CE\left[(\|\hat{x}\|^2 - \CE[\|\hat{x}\|^2])^2\right] \leq \sigma_b^2$, and $a_{t-1}^* \leq \bar{a}$ implied by the bounded fourth moment assumptions.
Substituting the above, we have
\begin{eqnarray*}
\lefteqn{\CE[s_t^2(a_t - a_t^*)^2]} \\
& \leq &  (1-\beta_{t-1})s_{t-1}^2(a_{t-1} - a_{t-1}^*)^2
+  \eta_{t-1}^2 L^2G^2\frac{\CE[s_{t}^2]}{\CE[\|\hat{x}\|^2]} + \frac{\eta_{t-1}^2}{\beta_{t-1}} L^2G^2\frac{\CE[s_{t}^2]}{\CE[\|\hat{x}\|^2]} 
\\
&& + 2\beta_{t-1}^2(\sigma_{\phi}^2+\bar{a}^2\sigma_b^2).
\end{eqnarray*}
To bound $\CE[s^2_t]$, we consider
\begin{eqnarray*}
\lefteqn{\CE[(s_t - \CE[\|\hat{x}\|^2])^2]} \\
 & = & (1 - \beta_{t-1})^2(s_{t-1}- \CE[\|\hat{x}\|^2])^2  + \beta_{t-1}^2\CE[(\|\hat{x}_{i_{t-1}}\|^2 - \CE[\|\hat{x}\|^2])^2] \\
 && +  2(1-\beta_{t-1})^2\beta_{t-1}(s_{t-1}- \CE[\|\hat{x}\|^2])\CE[(\|\hat{x}_{i_{t-1}}\|^2 - \CE[\|\hat{x}\|^2])] \\
  & = & (1 - \beta_{t-1})^2(s_{t-1}- \CE[\|\hat{x}\|^2])^2 + \beta_{t-1}^2\CE[(\|\hat{x}\|^2 - \CE[\|\hat{x}\|^2])^2] \\
    & \leq & (1 - \beta_{t-1})^2(s_{t-1}- \CE[\|\hat{x}\|^2])^2 + \beta_{t-1}^2\sigma_{b}^2 \\
\end{eqnarray*}
Let $\beta_t = 1/t^{c_2}$ for some constant $1/2 < c_2 \leq 1$. Taking total expectation over the history, and using induction, we can obtain
\begin{eqnarray*}
\CE[(s_t - \CE[\|\hat{x}\|^2])^2] \leq \frac{\sigma_b^2}{(t-1)^{c_2}},
\end{eqnarray*}
which leads $\CE[s_t^2] \leq \CE[\|\hat{x}\|^2]^2 + \frac{\sigma_b^2}{(t-1)^{c_2}}$. Then, combining it, we obtain
\begin{eqnarray*}
\lefteqn{\CE[s_t^2(a_t - a_t^*)^2]} \\
& \leq &  (1-\beta_{t-1})\CE[s_{t-1}^2(a_{t-1} - a_{t-1}^*)^2]
+ \eta_{t-1}^2 L^2G^2\CE[\|\hat{x}\|^2] + \eta_{t-1}^2 L^2G^2\frac{\sigma_b^2}{\CE[\|\hat{x}\|^2](t-1)^{c_2}} \\
&& + \frac{\eta_{t-1}^2}{\beta_{t-1}} L^2G^2\CE[\|\hat{x}\|^2]  + 
\eta_{t-1}^2L^2G^2\frac{\sigma_b^2}{\CE[\|\hat{x}\|^2]}  
 + 2\beta_{t-1}^2(\sigma_{\phi}^2+\bar{a}^2\sigma_b^2).
\end{eqnarray*}
Let use $\eta_t = O(1/t^{c_1})$ and $\beta_t = O(1/t^{c_2})$ with $1/2 < c_2 < c_1 \leq 1$. By induction, we have 
\begin{eqnarray*}
\CE[s_t^2(a_t - a_t^*)^2]
& \leq &  O\left(\max\left\{\frac{1}{t^{c_2}}, \frac{1}{t^{2(c_1 - c_2)}}\right\}\right).
\end{eqnarray*}
\end{proof} 


\subsection{Proof of Theorem~\ref{em_conv}}
\begin{proof} 
As $v_t$ is an unbiased estimator of $\nabla F(\theta_{t-1})$, by the smoothness of function, it follows that
\begin{eqnarray*}
\CE[F(\theta_{t})] - F(\theta_{t-1}) &\leq& -\eta_t\nabla F(\theta_{t-1})^T\CE[v_t] + \frac{1}{2}\eta_t^2L\CE[\|v_t\|^2] \\
&\leq &-\eta_t\|\nabla F(\theta_{t-1})\|^2 + \frac{1}{2}\eta_t^2L\CE[\|v_t\|^2] \\
&\leq &-(1 - \frac{1}{2}\eta_tL)\eta_t\|\nabla F(\theta_{t-1})\|^2 + \frac{1}{2}\eta_t^2L\CE[\|v_t - \nabla F(\theta_{t-1})\|^2] \\
&\leq &-(1 - \frac{1}{2}\eta_tL)\eta_t\|\nabla F(\theta_{t-1})\|^2 + \frac{1}{2}\eta_t^2L\sigma_a^2,
\end{eqnarray*}
where the last inequality follows from the Proposition~\ref{prop:em_var}. Let $\eta_t$ satisfies:
\begin{eqnarray*}
\eta_t = \frac{c}{\gamma + t} \text{ for some } c > \frac{1}{\mu} \text{ and } \gamma > 0 \text{ such that } \eta_1 \leq \frac{1}{L},
\end{eqnarray*}
which implies
\[\eta_tL \leq \eta_1L \leq 1.\]
Then, it follows that
\begin{eqnarray*}
\CE[F(\theta_{t})] - F(\theta_{t-1}) &\leq&
-(1 - \frac{1}{2}\eta_tL)\eta_t\|\nabla F(\theta_{t-1})\|^2 + \frac{1}{2}\eta_t^2L\sigma_a^2 \\
&\leq &-\frac{1}{2}\eta_t\|\nabla F(\theta_{t-1})\|^2 + \frac{1}{2}\eta_t^2L\sigma_a^2 \\
&\leq &-\eta_t\mu(F(\theta_{t-1}) - F(\theta_*)) + \frac{1}{2}\eta_t^2L\sigma_a^2.
\end{eqnarray*}
Subtracting $F(\theta_*)$ from both sides, taking total expectations, and rearranging, this yields
\begin{eqnarray*}
\CE[F(\theta_{t}) - F(\theta_*)] \leq (1 - \eta_t\mu)\CE[F(\theta_{t-1}) - F(\theta_*)] + \frac{1}{2}\eta_t^2L\sigma_a^2.
\end{eqnarray*}
We now prove the theorem by induction. Let $k = \gamma + t$. 
First, the definition of $\nu_1$ ensures that it holds for $t = 0$. 
Then, for $t \geq 1$, it follows that 
\begin{eqnarray*}
\CE[F(\theta_{t}) - F(\theta_*)] & \leq & \left(1 - \frac{c\mu}{k}\right)\frac{\nu_1}{k} + \frac{c^2L\sigma_a^2}{2k^2} \\
& = & \left(\frac{k - c\mu}{k^2}\right)\nu_1 + \frac{c^2L\sigma_a^2}{2k^2} \\
& = & \left(\frac{k - 1}{k^2}\right)\nu_1 - \frac{c\mu - 1}{k^2}\nu_1 + \frac{c^2L\sigma_a^2}{2k^2} \\
& \leq & \left(\frac{k - 1}{k^2}\right)\nu_1 \\
& \leq & \frac{\nu_1}{k + 1},
\end{eqnarray*}
where the second-to-last inequality holds by the definition of $\nu_1$. The last inequality holds as $k^2 \geq (k+1)(k-1)$.

\end{proof} 


\subsection{Proof of Theorem~\ref{em_a_conv}}
\begin{proof} 
 Let us proceed by expanding $\CE[\|\theta_t - \theta_*\|^2]$:
\begin{eqnarray*}
\lefteqn{\CE[\|\theta_t - \theta_*\|^2]}\\
& = & \|\theta_{t-1} - \theta_*\|^2 - 2\eta_t\langle \nabla F(\theta_{t-1}), \theta_{t-1} - \theta_*\rangle + \eta_t^2\CE[\|v_t\|^2] \\
& \leq & (1 - \eta_t\mu)\|\theta_{t-1} - \theta_*\|^2 - 2\eta_t(F(\theta_{t-1}) - F(\theta_*))  + \eta_t^2\CE[\|v_t\|^2] \\
& \leq & (1 - \eta_t\mu)\|\theta_{t-1} - \theta_*\|^2 - 2\eta_t(F(\theta_{t-1}) - F(\theta_*)) + \eta_t^2\CE[\|v_t - \nabla F(\theta_{t-1})\|^2] + \eta_t^2\|\nabla F(\theta_{t-1})\|^2\\
& \leq & (1 - \eta_t\mu)\|\theta_{t-1} - \theta_*\|^2 - (1 - \eta_tL)2\eta_t(F(\theta_{t-1}) - F(\theta_*)) + \eta_t^2\sigma_a^2.
\end{eqnarray*}
When $\eta_t \leq 1/(2L)$, we have $(1 - \eta_tL)2\eta_t \geq \eta_t$.  
Taking total expectations and rearranging gives
\begin{eqnarray*}
\eta_t\CE[F(\theta_{t-1}) - F(\theta_*)]
& \leq & (1 - \eta_t\mu)\CE[\|\theta_{t-1} - \theta_*\|^2] - \CE[\|\theta_t - \theta_*\|^2] + \eta_t^2\sigma_a^2.
\end{eqnarray*}
Let $\eta_t = 2/(\mu(\gamma + t))$. Dividing by $\eta_t$ yields
\begin{eqnarray*}
\lefteqn{\CE[F(\theta_{t-1}) - F(\theta_*)]}\\
& \leq & (\frac{1}{\eta_t} - \mu)\CE[\|\theta_{t-1} - \theta_*\|^2] - \frac{1}{\eta_t}\CE[\|\theta_t - \theta_*\|^2] + \eta_t\sigma_a^2\\
& = & (\gamma + t - 2)\frac{\mu}{2}\CE[\|\theta_{t-1} - \theta_*\|^2] - (\gamma+t)\frac{\mu}{2}\CE[\|\theta_t - \theta_*\|^2]  +  \frac{2\sigma_a^2}{\mu (\gamma + t)}.
\end{eqnarray*}
Multiplying by $\gamma + t - 1$ yields
\begin{eqnarray*}
\lefteqn{(\gamma+t-1)\CE[F(\theta_{t-1}) - F(\theta_*)]}\\
& \leq & (\gamma+t-1)(\gamma + t - 2)\frac{\mu}{2}\CE[\|\theta_{t-1} - \theta_*\|^2] - (\gamma+t)(\gamma+t-1)\frac{\mu}{2}\CE[\|\theta_t - \theta_*\|^2] + \frac{(\gamma+t-1)}{\gamma + t}\frac{2\sigma_a^2}{\mu}\\
& \leq & (\gamma+t-1)(\gamma + t - 2)\frac{\mu}{2}\CE[\|\theta_{t-1} - \theta_*\|^2] - (\gamma+t)(\gamma+t-1)\frac{\mu}{2}\CE[\|\theta_t - \theta_*\|^2]  + \frac{2\sigma_a^2}{\mu}.
\end{eqnarray*}
By summing the above inequality from $t = 1$ to $t= T$, we have a telescoping sum that simplifies as follows:
\begin{eqnarray*}
\lefteqn{\sum_{t=1}^T(\gamma+t-1)\CE[F(\theta_{t-1}) - F(\theta_*)]}\\
& \leq & \gamma(\gamma - 1)\frac{\mu}{2}\|\theta_0 - \theta_*\|^2 - (\gamma+T)(\gamma+T-1)\frac{\mu}{2}\CE[\|\theta_T- \theta_*\|^2]  + \frac{2T\sigma_a^2}{\mu} \\
& \leq & \gamma(\gamma - 1)\frac{\mu}{2}\|\theta_0 - \theta_*\|^2   + \frac{2T\sigma_a^2}{\mu}.
\end{eqnarray*}
Dividing by $\sum_{t=1}^T(\gamma+t-1) = (2T\gamma + T(T-1))/2$ and using Jensen's inequality yields
\begin{eqnarray*}
\CE[F(\bar{\theta}_T) - F(\theta_*)]
& \leq & \frac{\mu\gamma(\gamma - 1)}{T(2\gamma + T-1)}\|\theta_0 - \theta_*\|^2  + \frac{4\sigma_a^2}{\mu(2\gamma + T - 1)}.
\end{eqnarray*}
\end{proof} 


\subsection{Proof of Theorem~\ref{em_p_conv}}

Let us define
\[q_t = \frac{1}{\eta_t}(\theta_{t-1} - \theta_t)\]

\begin{lemma} (\cite{xiao2014proximal}, Lemma~3) \label{lemma:proximal_recursive}
If $\eta_t \leq 1/L$ for all $t$, then we have for any $\theta \in \R^d$,
\begin{eqnarray*}
P(\theta) \geq P(\theta_t) + \langle q_t, \theta - \theta_{t-1}\rangle + \frac{\eta_t}{2}\|q_t\|^2 + \frac{\mu}{2}\|\theta_{t-1} - \theta\|^2 + \langle v_t - \nabla F(\theta_{t-1}), \theta_t - \theta\rangle
\end{eqnarray*}
for all $t$.
\end{lemma}

\begin{proof} (proof of Theorem~\ref{em_p_conv})
Let $\eta_t$ satisfies:
\begin{eqnarray*}
\eta_t = \frac{c}{\gamma + t} \text{ for some } c > \frac{1}{\mu} \text{ and } \gamma > 0 \text{ such that } \eta_1 \leq \frac{1}{L},
\end{eqnarray*}
 Let us proceed by expanding $\CE[\|\theta_t - \theta_*\|^2]$:
\begin{eqnarray}
\lefteqn{\|\theta_t - \theta_*\|^2} \nonumber\\
& = & \|\theta_{t-1} - \theta_*\|^2 - 2\eta_t\langle q_t, \theta_{t-1} - \theta_*\rangle + \eta_t^2\|q_t\|^2 \nonumber\\
& \leq & (1 - \eta_t\mu)\|\theta_{t-1} - \theta_*\|^2 - 2\eta_t[P(\theta_{t}) - P(\theta_*)] - 2\eta_t\langle v_t - \nabla F(\theta_{t-1}), \theta_t - \theta_*\rangle, \label{eq:recursive_p1}
\end{eqnarray}
where the inequality follows from the Lemma~\ref{lemma:proximal_recursive} with $\theta = \theta_*$.  Define the proximal full gradient update as
\[\hat{\theta}_t = \text{prox}_{\eta_t r}(\theta_{t-1} - \eta_t \nabla F(\theta_{t-1})).\]
Then, we have
\begin{eqnarray*}
\lefteqn{-\langle v_t - \nabla F(\theta_{t-1}), \theta_t - \theta_*\rangle}\\
& =& -\langle v_t - \nabla F(\theta_{t-1}), \theta_t - \hat{\theta}_t\rangle -\langle v_t - \nabla F(\theta_{t-1}), \hat{\theta}_t - \theta_*\rangle \\
& \leq & \| v_t - \nabla F(\theta_{t-1})\|\|\theta_t - \hat{\theta}_t\| - \langle v_t - \nabla F(\theta_{t-1}), \hat{\theta}_t - \theta_*\rangle \\
& \leq & \eta_t\| v_t - \nabla F(\theta_{t-1})\|^2 - \langle v_t - \nabla F(\theta_{t-1}), \hat{\theta}_t - \theta_*\rangle,
\end{eqnarray*}
where the last inequality follows from the non-expansiveness of proximal operators. Taking expectation and combining with (\ref{eq:recursive_p1}), we obtain
\begin{eqnarray}
\lefteqn{\CE[\|\theta_t - \theta_*\|^2]} \nonumber\\
& \leq & (1 - \eta_t\mu)\|\theta_{t-1} - \theta_*\|^2 - 2\eta_t\CE[P(\theta_{t}) - P(\theta_*)] + 2\eta_t^2\CE[\|v_t - \nabla F(\theta_{t-1})\|^2] - 2\eta_t\langle \CE[v_t] - \nabla F(\theta_{t-1}), \hat{\theta}_t - \theta_*\rangle \nonumber\\
& \leq & (1 - \eta_t\mu)\|\theta_{t-1} - \theta_*\|^2 - 2\eta_t\CE[P(\theta_{t}) - P(\theta_*)] + 2\eta_t^2\CE[\|v_t - \nabla F(\theta_{t-1})\|^2] \label{eq:p_recursive_n} \\
& \leq & (1 - \eta_t\mu)\|\theta_{t-1} - \theta_*\|^2 - 2\eta_t\CE[P(\theta_{t}) - P(\theta_*)] + 2\eta_t^2\sigma_a^2 \label{eq:p_recursive} \\
& \leq & (1 - \eta_t\mu)\|\theta_{t-1} - \theta_*\|^2 - \eta_t\mu\CE[\|\theta_t - \theta_*\|^2] + 2\eta_t^2\sigma_a^2. \nonumber
\end{eqnarray}
The second-to-last inequality holds due to the Proposition~\ref{prop:em_var}. Taking total expectations, 
rearranging, this yields
\begin{eqnarray*}
\CE[\|\theta_t - \theta_*\|^2] \leq \frac{1 - \eta_t\mu}{1 + \eta_t\mu}\CE[\|\theta_{t- 1}- \theta_*\|^2]+ \frac{2\eta_t^2\sigma_a^2}{1 + \eta_t\mu}.
\end{eqnarray*}
We now prove the theorem by induction. Let $k = \gamma + t$. The definition of $\nu_5$ ensures that it holds for $t = 0$. 
Then, for $t \geq 1$,  
\begin{eqnarray*}
\CE[\|\theta_t - \theta_*\|^2] & \leq & \frac{k - c\mu}{k + c\mu}\frac{\nu_5}{k} + \frac{2c^2\sigma_a^2}{k(k + c\mu)} \\
& = & \frac{\nu_5}{k + 1} - \frac{2c\mu - 1 + c\mu/k}{(k+1)(k+c\mu)}\nu_5 + \frac{2c^2\sigma_a^2}{k(k + c\mu)} \\
& \leq & \frac{\nu_5}{k + 1},
\end{eqnarray*}
where the second-to-last inequality holds by the definition of $\nu$. 
\end{proof} 


\subsection{Proof of Theorem~\ref{em_ap_conv}}
\begin{proof} 
Following (\ref{eq:p_recursive}), taking total expectations, and rearranging, we have 
\begin{eqnarray*}
2\eta_t\CE[P(\theta_{t}) - F(\theta_*)]
& \leq & (1 - \eta_t\mu)\CE[\|\theta_{t-1} - \theta_*\|^2] - \CE[\|\theta_t - \theta_*\|^2] + 2\eta_t^2\sigma_a^2.
\end{eqnarray*}
Let $\eta_t = 2/(\mu(\gamma + t))$. Dividing by $2\eta_t$ yields
\begin{eqnarray*}
\lefteqn{\CE[P(\theta_{t}) - P(\theta_*)]}\\
& \leq & (\frac{1}{2\eta_t} - \frac{\mu}{2})\CE[\|\theta_{t-1} - \theta_*\|^2] - \frac{1}{2\eta_t}\CE[\|\theta_t - \theta_*\|^2] + \eta_t\sigma_a^2\\
& = & (\gamma + t - 2)\frac{\mu}{4}\CE[\|\theta_{t-1} - \theta_*\|^2] - (\gamma+t)\frac{\mu}{4}\CE[\|\theta_t - \theta_*\|^2]  +  \frac{2\sigma_a^2}{\mu (\gamma + t)}.
\end{eqnarray*}
Multiplying by $\gamma + t - 1$ yields
\begin{eqnarray*}
\lefteqn{(\gamma+t-1)\CE[P(\theta_{t}) - P(\theta_*)]}\\
& \leq & (\gamma+t-1)(\gamma + t - 2)\frac{\mu}{4}\CE[\|\theta_{t-1} - \theta_*\|^2] - (\gamma+t)(\gamma+t-1)\frac{\mu}{4}\CE[\|\theta_t - \theta_*\|^2] + \frac{(\gamma+t-1)}{\gamma + t}\frac{2\sigma_a^2}{\mu}\\
& \leq & (\gamma+t-1)(\gamma + t - 2)\frac{\mu}{4}\CE[\|\theta_{t-1} - \theta_*\|^2] - (\gamma+t)(\gamma+t-1)\frac{\mu}{4}\CE[\|\theta_t - \theta_*\|^2] + \frac{2\sigma_a^2}{\mu}.
\end{eqnarray*}
By summing the above inequality from $t = 1$ to $t= T$, we have a telescoping sum that simplifies as follows:
\begin{eqnarray*}
\lefteqn{\sum_{t=1}^T(\gamma+t-1)\CE[P(\theta_{t}) - P(\theta_*)]}\\
& \leq & \gamma(\gamma - 1)\frac{\mu}{4}\|\theta_0 - \theta_*\|^2 - (\gamma+T)(\gamma+T-1)\frac{\mu}{4}\CE[\|\theta_T- \theta_*\|^2]  + \frac{2T\sigma_a^2}{\mu} \\
& \leq & \gamma(\gamma - 1)\frac{\mu}{4}\|\theta_0 - \theta_*\|^2   + \frac{2T\sigma_a^2}{\mu}.
\end{eqnarray*}
Dividing by $\sum_{t=1}^T(\gamma+t-1) = (2T\gamma + T(T-1))/2$ and using Jensen's inequality yields
\begin{eqnarray*}
\CE[P(\bar{\theta}_T) - P(\theta_*)]
& \leq & \frac{\mu\gamma(\gamma - 1)}{2T(2\gamma + T - 1)}\|\theta_0 - \theta_*\|^2  + \frac{4\sigma_a^2}{\mu(2\gamma + T - 1)}.
\end{eqnarray*}
\end{proof}


\subsection{Proof of Theorem~\ref{em_conv_ssaga}}
In the following, let us denote $a_{i_t}^t = \phi_{i_t}(\langle \varphi_{i_t}^t, \vartheta_{i_t}^t\rangle)$, where $\varphi_{i_t}^t$ and $\vartheta_{i_t}^t$ denote the old noisy sample and parameter used to compute $a_{i_t}^t$, respectively. Note that $\vartheta_i^1 = \theta_0$ for all $i$. Let define $f_i(\theta; \xi_i) = \phi_{i}(\hat{x}_i^T\theta) + g(\theta)$, and $f_i(\theta) \equiv \CE_{\xi_i}[\phi_i(\hat{x}_i^T\theta)] + g(\theta)$. In the following, we prove the convergence on the composite problem (\ref{eq:em_com_problem}).

\begin{lemma} \label{lemma:variance_asaga}
For any $\rho_t > 0$ , $\forall t$, we have
\begin{eqnarray*} 
\lefteqn{\CE[\|v_t - \nabla F(\theta_{*})\|^2]}\\
   & \leq & 2(1 + \rho_t^{-1})\CE\left[\left\|(\phi_{i_t}'(\langle \varphi_{i_t}^t, \vartheta_{i_t}^t\rangle)) -  \phi_{i_t}'(\hat{x}_{i_t} ^T\vartheta_{i_t}^t))\hat{x}_{i_t}\right\|^2\right] + 2(1 + \rho_t^{-1})\CE\left[\left\|(\phi_{i_t}'(\hat{x}_{i_t} ^T\vartheta_{i_t}^t)) -  \phi_{i_t}'(\hat{x}_{i_t} ^T\theta_{*}))\hat{x}_{i_t}\right\|^2\right]  \\
  &&+ (1 + \rho_t)\CE\left[\left\|\nabla f_{i_t}(\theta_{t-1}; \xi_{i_t}) - \nabla f_{i_t}(\theta_*; \xi_{i_t})\right\|^2\right]
 - \rho_t\|\nabla F(\theta_{t-1}) -  \nabla F(\theta_{*})\|^2.
 \end{eqnarray*}
\end{lemma}
\begin{proof} 
Note that $\CE[\phi_{i_t}'(\hat{x}_{i_t}^T\theta)\hat{x}_{i_t}] = \nabla F(\theta) - \nabla g(\theta)$. We follow a similar argument as in the SAGA proof \cite{defazio-14} for this term.
\begin{eqnarray*} 
\lefteqn{\CE[\|v_t - \nabla F(\theta_{*})\|^2]}\\
& = & \CE\left[\left\|(\phi_{i_t}'(\hat{x}_{i_t} ^T\theta_{t-1}) - a_{i_t}^t)\hat{x}_{i_t}  + \frac{1}{n}\sum_{i=1}^na_{i}^t\bar{x}_i + \nabla g(\theta_{t-1})  -  \nabla F(\theta_{*})\right\|^2\right] \\
& = & \CE\left[\left\|\left[(a_{i_t}^t -  \phi_{i_t}'(\hat{x}_{i_t} ^T\theta_{*}))\hat{x}_{i_t} - \frac{1}{n}\sum_{i=1}^na_{i}^t\bar{x}_i + \nabla F(\theta_{*}) - \nabla g(\theta_*)\right]  - [\nabla f_{i_t}(\theta_{t-1}; \xi_{i_t}) - \nabla f_{i_t}(\theta_*; \xi_{i_t})- \nabla F(\theta_{t-1}) +  \nabla F(\theta_{*})]\right\|^2\right] \\
&& + \|\nabla F(\theta_{t-1}) -  \nabla F(\theta_{*})\|^2 \\
& \leq & (1 + \rho_t^{-1})\CE\left[\left\|(a_{i_t}^t -  \phi_{i_t}'(\hat{x}_{i_t} ^T\theta_{*}))\hat{x}_{i_t} - \frac{1}{n}\sum_{i=1}^na_{i}^t\bar{x}_i + \nabla F(\theta_{*}) - \nabla g(\theta_*)\right\|^2\right] \\
&& + (1 + \rho_t)\CE\left[\left\|\nabla f_{i_t}(\theta_{t-1}; \xi_{i_t}) - \nabla f_{i_t}(\theta_*; \xi_{i_t}) - \nabla F(\theta_{t-1}) +  \nabla F(\theta_{*})\right\|^2\right]
 + \|\nabla F(\theta_{t-1}) -  \nabla F(\theta_{*})\|^2 \\
 & \leq & (1 + \rho_t^{-1})\CE\left[\left\|(a_{i_t}^t -  \phi_{i_t}'(\hat{x}_{i_t} ^T\theta_{*}))\hat{x}_{i_t}\right\|^2\right]  + (1 + \rho_t)\CE\left[\left\|\nabla f_{i_t}(\theta_{t-1}; \xi_{i_t}) - \nabla f_{i_t}(\theta_*; \xi_{i_t})\right\|^2\right]
 - \rho_t\|\nabla F(\theta_{t-1}) -  \nabla F(\theta_{*})\|^2 \\
  & \leq & 2(1 + \rho_t^{-1})\CE\left[\left\|(\phi_{i_t}'(\langle \varphi_{i_t}^t, \vartheta_{i_t}^t\rangle)) -  \phi_{i_t}'(\hat{x}_{i_t} ^T\vartheta_{i_t}^t))\hat{x}_{i_t}\right\|^2\right]+ 2(1 + \rho_t^{-1})\CE\left[\left\|(\phi_{i_t}'(\hat{x}_{i_t} ^T\vartheta_{i_t}^t)) -  \phi_{i_t}'(\hat{x}_{i_t} ^T\theta_{*}))\hat{x}_{i_t}\right\|^2\right]  \\
  &&+ (1 + \rho_t)\CE\left[\left\|\nabla f_{i_t}(\theta_{t-1}; \xi_{i_t}) - \nabla f_{i_t}(\theta_*; \xi_{i_t})\right\|^2\right]
 - \rho_t\|\nabla F(\theta_{t-1}) -  \nabla F(\theta_{*})\|^2.
\end{eqnarray*}
\end{proof} 

\begin{lemma}  \label{lemma:smooth_asaga}
\begin{eqnarray*} 
\frac{1}{n}\sum_{i=1}^n\CE_{\xi_i}\left[\left\|(\phi_{i}'(\hat{x}_{i} ^T\theta)) -  \phi_i'(\hat{x}_{i} ^T\theta_{*}))\hat{x}_{i}\right\|^2\right]   \leq 
2L\left[\frac{1}{n}\sum_{i=1}^nf_i(\theta) - F(\theta_*) - \frac{1}{n}\sum_{i=1}^n\langle \nabla f_i(\theta_*), \theta - \theta_* \rangle\right].
 \end{eqnarray*}
\end{lemma} 
\begin{proof} 
The smoothness of $\phi_i$ implies that $\| \phi_i'(\hat{x}_i^T\theta)\hat{x}_i - \phi_i'(\hat{x}_i^T\theta_*)\hat{x}_i\|^2 \leq 2L[\phi_i(\hat{x}_i^T\theta) - \phi_i(\hat{x}_i^T\theta_*) - \langle \phi_i'(\hat{x}_i^T\theta_*)\hat{x}_i, \theta - \theta' \rangle]$. We can then obtain the result by taking expectation over the random variable $\xi_i$, averaging over $n$ functions, and adding nonnegative term $2L[g(\theta) - g(\theta_*) - \langle \nabla g(\theta_*), \theta - \theta_* \rangle]$ due to the convexity of $g$.
\end{proof} 

\begin{lemma} (Lemma~1 of \cite{defazio-14})  \label{lemma:strong_asaga}
Suppose that each $f_i(\theta; \xi_i)$  is $\mu$-strongly convex, then for all $\theta$ and $\theta_*$:
\begin{eqnarray*} 
\langle \nabla F(\theta), \theta_* - \theta \rangle \leq \frac{L - \mu}{L}[F(\theta_*) - F(\theta)] - \frac{\mu}{2}\|\theta - \theta_*\|^2 - \frac{1}{2nL}\sum_{i=1}^n\CE_{\xi_i}[\|\nabla f_{i}(\theta; \xi_{i}) - \nabla f_{i}(\theta_*; \xi_{i})\|^2] - \frac{\mu}{L}\langle \nabla F(\theta_*), \theta - \theta_*\rangle.
 \end{eqnarray*}
\end{lemma} 

\begin{proof}  (proof of Theorem~\ref{em_conv_ssaga})

Expanding $\CE[\|\theta_t - \theta_*\|^2]$:
\begin{eqnarray*}
\lefteqn{\CE[\|\theta_t - \theta_*\|^2]} \nonumber\\
& \leq & \CE[\|\text{prox}_{\eta_tr}(\theta_{t-1} - \eta_t v_t) - \text{prox}_{\eta_tr}(\theta_* - \eta_t\nabla F(\theta_*))\|^2] \\
& \leq & \CE[\|\theta_{t-1} - \eta_t v_t - \theta_* + \eta_t\nabla F(\theta_*)\|^2] \\
& \leq & \|\theta_{t-1} - \theta_*\|^2  - 2\eta_t\langle\nabla F(\theta_{t-1}) - \nabla F(\theta_*), \theta_{t-1} - \theta_*\rangle+ \eta_t^2\CE[\|v_t - \nabla F(\theta_*)\|^2].
\end{eqnarray*}
Then, by combining Lemma~\ref{lemma:variance_asaga}, we obtain
\begin{eqnarray*}
\lefteqn{\CE[\|\theta_t - \theta_*\|^2]} \nonumber\\
& \leq & \|\theta_{t-1} - \theta_*\|^2  - 2\eta_t\langle\nabla F(\theta_{t-1}) - \nabla F(\theta_*), \theta_{t-1} - \theta_*\rangle+ 2\eta_t^2(1 + \rho_t^{-1})\CE\left[\left\|(\phi_{i_t}'(\langle \varphi_{i_t}^t, \vartheta_{i_t}^t\rangle)) -  \phi_{i_t}'(\hat{x}_{i_t} ^T\vartheta_{i_t}^t))\hat{x}_{i_t}\right\|^2\right] 
\\
&& - \eta_t^2\rho_t\|\nabla F(\theta_{t-1}) -  \nabla F(\theta_{*})\|^2 
 + 2\eta_t^2(1 + \rho_t^{-1})\CE\left[\left\|(\phi_{i_t}'(\hat{x}_{i_t} ^T\vartheta_{i_t}^t)) -  \phi_{i_t}'(\hat{x}_{i_t} ^T\theta_{*}))\hat{x}_{i_t}\right\|^2\right]  
\\
&& + \eta_t^2(1 + \rho_t)\CE\left[\left\|\nabla f_{i_t}(\theta_{t-1}; \xi_{i_t}) - \nabla f_{i_t}(\theta_*; \xi_{i_t})\right\|^2\right]\\
& \leq & \|\theta_{t-1} - \theta_*\|^2  - 2\eta_t\langle\nabla F(\theta_{t-1}) - \nabla F(\theta_*), \theta_{t-1} - \theta_*\rangle+ 2\eta_t^2(1 + \rho_t^{-1})\CE\left[\left\|(\phi_{i_t}'(\langle \varphi_{i_t}^t, \vartheta_{i_t}^t\rangle)) -  \phi_{i_t}'(\hat{x}_{i_t} ^T\vartheta_{i_t}^t))\hat{x}_{i_t}\right\|^2\right]
\\
&& - \eta_t^2\rho_t\|\nabla F(\theta_{t-1}) -  \nabla F(\theta_{*})\|^2  
+ 4\eta_t^2L(1 + \rho_t^{-1})\left[\frac{1}{n}\sum_{i=1}^nf(\vartheta_{i}^t) - F(\theta_*) - \frac{1}{n}\sum_{i=1}^n\langle \nabla f_i(\theta_*), \vartheta_{i}^t - \theta_* \rangle\right]
\\
&& + \eta_t^2(1 + \rho_t)\CE\left[\left\|\nabla f_{i_t}(\theta_{t-1}; \xi_{i_t}) - \nabla f_{i_t}(\theta_*; \xi_{i_t})\right\|^2\right], 
\end{eqnarray*}
where the last inequality holds by applying Lemma~\ref{lemma:smooth_asaga}. Using Lemma~\ref{lemma:strong_asaga} with $\theta = \theta_{t-1}$, we have
\begin{eqnarray*}
\lefteqn{\CE[\|\theta_t - \theta_*\|^2]} \nonumber\\
& \leq & (1 - \eta_t\mu)\|\theta_{t-1} - \theta_*\|^2 + 2\eta_t^2(1 + \rho_t^{-1})\CE\left[\left\|(\phi_{i_t}'(\langle \varphi_{i_t}^t, \vartheta_{i_t}^t\rangle)) -  \phi_{i_t}'(\hat{x}_{i_t} ^T\vartheta_{i_t}^t))\hat{x}_{i_t}\right\|^2\right] 
-\eta_t^2\rho_t\|\nabla F(\theta_{t-1}) -  \nabla F(\theta_{*})\|^2 \\
&&
- \frac{2\eta_t(L -\mu)}{L} [F(\theta_{t-1}) - F(\theta_*) - \langle \nabla F(\theta_*), \theta_{t-1} - \theta_*\rangle]
+ (\eta_t^2(1 + \rho_t) - \frac{\eta_t}{L})\CE\left[\left\|\nabla f_{i_t}(\theta_{t-1}; \xi_{i_t}) - \nabla f_{i_t}(\theta_*; \xi_{i_t})\right\|^2\right]
\\
&& 
+ 4\eta_t^2L(1 + \rho_t^{-1})\left[\frac{1}{n}\sum_{i=1}^nf(\vartheta_{i}^t) - F(\theta_*) - \frac{1}{n}\sum_{i=1}^n\langle \nabla f_i(\theta_*), \vartheta_{i}^t - \theta_* \rangle\right].
\end{eqnarray*}
Define Lyapunov function $C_t$ as
\begin{eqnarray*}
C_t = b\|\theta_t - \theta_*\|^2 + \alpha_t\left[\frac{1}{n}\sum_{i=1}^nf_i(\vartheta_{i}^{t+1}) - F(\theta_*) - \frac{1}{n}\sum_{i=1}^n\langle \nabla f_i(\theta_*), \vartheta_{i}^{t+1} - \theta_* \rangle\right].
\end{eqnarray*}
for some constant $b$, and we have
\begin{eqnarray*}
\lefteqn{\CE\left[\frac{1}{n}\sum_{i=1}^nf_i(\vartheta_{i}^{t+1}) - F(\theta_*) - \frac{1}{n}\sum_{i=1}^n\langle \nabla f_i(\theta_*), \vartheta_{i}^{t+1} - \theta_* \rangle\right] }\\
& = & \frac{1}{n}[F(\theta_{t-1}) - F(\theta_*) - \langle \nabla F(\theta_*), \theta_{t-1} - \theta_*\rangle ] \\
&& + (1 - \frac{1}{n})\left[\frac{1}{n}\sum_{i=1}^nf_i(\vartheta_{i}^t) - F(\theta_*) - \frac{1}{n}\sum_{i=1}^n\langle \nabla f_i(\theta_*), \vartheta_{i}^t - \theta_* \rangle\right].
\end{eqnarray*}
Then, we obtain
\begin{eqnarray}
\lefteqn{\CE[C_t]} \nonumber\\
& \leq & (1 - \eta_t\mu)C_{t-1} + 2b\eta_t^2(1 + \rho_t^{-1})\CE\left[\left\|(\phi_{i_t}'(\langle \varphi_{i_t}^t, \vartheta_{i_t}^t\rangle)) -  \phi_{i_t}'(\hat{x}_{i_t} ^T\vartheta_{i_t}^t))\hat{x}_{i_t}\right\|^2\right]   \nonumber \\
&& + (\frac{\alpha_t}{n} - \frac{2b\eta_t(L -\mu)}{L} - 2b\eta_t^2\mu\rho_t)[F(\theta_{t-1}) - F(\theta_*) - \langle \nabla F(\theta_*), \theta_{t-1} - \theta_*\rangle] \nonumber
\\ 
&&
+ \eta_tb(\eta_t(1 + \rho_t) - \frac{1}{L})\CE\left[\left\|\nabla f_{i_t}(\theta_{t-1}; \xi_{i_t}) - \nabla f_{i_t}(\theta_*; \xi_{i_t})\right\|^2\right] \nonumber
\\
&& 
+ (\alpha_{t-1}\eta_t\mu + \alpha_t + 4b\eta_t^2L(1 + \rho_t^{-1}) -\alpha_{t-1} - \frac{\alpha_t}{n} )\left[\frac{1}{n}\sum_{i=1}^nf_i(\vartheta_{i}^t) - F(\theta_*) - \frac{1}{n}\sum_{i=1}^n\langle \nabla f_i(\theta_*), \vartheta_{i}^t - \theta_* \rangle\right] \label{eq:ct}.
\end{eqnarray}
Let us assume that
\begin{eqnarray*}
\eta_t & = & \frac{c}{\gamma + t} \text{ for some } c > \frac{1}{\mu} \text{ and } \gamma > 0 \text{ such that } \eta_1 \leq \frac{1}{3(\mu n + L)}, \\
\rho_t & = & \frac{\gamma+t}{cL} -  1, \\
b & = & \frac{1}{2n}, \\
\alpha_t &=& \eta_t(1 - \eta_t\mu) \text{ and } \alpha_0 = \alpha_1.
\end{eqnarray*}
Then, we obtain
\begin{eqnarray*}
\frac{\alpha_t}{n} - \frac{2b\eta_t(L -\mu)}{L} - 2b\eta_t^2\mu\rho_t = 0, \\
\eta_t(1 + \rho_t) - \frac{1}{L} = 0, \\
\alpha_{t-1}\eta_t\mu + \alpha_t + 4b\eta_t^2L(1 + \rho_t^{-1}) -\alpha_{t-1} - \frac{\alpha_t}{n}  \leq 0.
\end{eqnarray*}
Hence, 
\begin{eqnarray*}
\CE[C_t]
& \leq & (1 - \eta_t\mu)C_{t-1} + \frac{1}{n}\eta_t^2(1 + \rho_t^{-1})\CE\left[\left\|(\phi_{i_t}'(\langle \varphi_{i_t}^t, \vartheta_{i_t}^t\rangle)) -  \phi_{i_t}'(\hat{x}_{i_t} ^T\vartheta_{i_t}^t))\hat{x}_{i_t}\right\|^2\right] \\
& \leq & (1 - \eta_t\mu)C_{t-1} + \frac{1}{n}\eta_t^2(1 + \frac{L}{3\mu n + 2L})\CE\left[\left\|(\phi_{i_t}'(\langle \varphi_{i_t}^t, \vartheta_{i_t}^t\rangle)) -  \phi_{i_t}'(\hat{x}_{i_t} ^T\vartheta_{i_t}^t))\hat{x}_{i_t}\right\|^2\right] \\
& \leq & (1 - \eta_t\mu)C_{t-1} + \frac{2\eta_t^2}{n}\CE\left[\left\|(\phi_{i_t}'(\langle \varphi_{i_t}^t, \vartheta_{i_t}^t\rangle)) -  \phi_{i_t}'(\hat{x}_{i_t} ^T\vartheta_{i_t}^t))\hat{x}_{i_t}\right\|^2\right].
\end{eqnarray*}
Taking total expectation, we have
\begin{eqnarray*}
\CE[C_t]
& \leq & (1 - \eta_t\mu)C_{t-1} + \frac{2\eta_t^2}{n}\CE\left[\left\|(\phi_{i_t}'(\langle \varphi_{i_t}^t, \vartheta_{i_t}^t\rangle)) -  \phi_{i_t}'(\hat{x}_{i_t} ^T\vartheta_{i_t}^t))\hat{x}_{i_t}\right\|^2\right] \\
& = & (1 - \eta_t\mu)C_{t-1} + \frac{2\eta_t^2}{n}\frac{1}{n}\sum_{i=1}^n\CE\left[\CE_{\xi_i,\xi_i'}\left[\left\|(\phi_{i}'(\langle \hat{x}_i', \vartheta_{i}^t\rangle)) -  \phi_{i}'(\hat{x}_{i} ^T\vartheta_{i}^t))\hat{x}_{i}\right\|^2\right] | \xi_i,\xi_i' \right] \\
& \leq & (1 - \eta_t\mu)\CE[C_{t-1}] + \frac{2\eta_t^2\sigma_c^2}{n},
\end{eqnarray*}
where the second-to-last equality holds as $\vartheta_{i}^t$ is independent from both $\varphi_{i}^t$ and $\hat{x}_{i}$ for each $i$, 
and $\vartheta_{i}^t \in \{\theta_0, \dots, \theta_{t-2}\}$.
Similar to the proof of Theorem~\ref{em_conv}, we can prove following by induction:
\begin{eqnarray*}
\CE[C_t] \leq \frac{\omega}{k + 1},
\end{eqnarray*}
where 
\begin{eqnarray*}
\omega \equiv \max\left(\frac{2c^2\sigma_c^2}{n(c\mu - 1)}, (\gamma + 1)C_0\right).
\end{eqnarray*}
Then, we have
\begin{eqnarray*}
\CE[\|\theta_t - \theta_*\|^2] \leq \frac{\nu_2}{k + 1},
\end{eqnarray*}
\begin{eqnarray*}
\nu_2 \equiv \max\left(\frac{4c^2\sigma_c^2}{c\mu - 1}, (\gamma + 1)C\right),
\end{eqnarray*}
where 
\[C \equiv \|\theta_0 - \theta_*\|^2  + \frac{2n}{3(\mu n + L)}\left[F(\theta_0) - F(\theta_*) - \langle \nabla F(\theta_*), \theta_0 - \theta_* \rangle\right].\]
\end{proof}


\subsection{Proof of Theorem~\ref{em_a_conv_ssaga}}
\begin{proof} 
Following (\ref{eq:p_recursive_n}), we have
\begin{eqnarray*}
\lefteqn{\CE[\|\theta_t - \theta_*\|^2]} \nonumber\\
& \leq & (1 - \eta_t\mu)\|\theta_{t-1} - \theta_*\|^2 - 2\eta_t\CE[P(\theta_{t}) - P(\theta_*)] + 2\eta_t^2\CE[\|v_t - \nabla F(\theta_{t-1})\|^2].
\end{eqnarray*}
A small change of the argument in Lemma~\ref{lemma:variance_asaga} leads to
\begin{eqnarray*} 
\lefteqn{\CE[\|v_t - \nabla F(\theta_{t-1})\|^2]}\\
   & \leq & 2(1 + \rho_t^{-1})\CE\left[\left\|(\phi_{i_t}'(\langle \varphi_{i_t}^t, \vartheta_{i_t}^t\rangle)) -  \phi_{i_t}'(\hat{x}_{i_t} ^T\vartheta_{i_t}^t))\hat{x}_{i_t}\right\|^2\right] + 2(1 + \rho_t^{-1})\CE\left[\left\|(\phi_{i_t}'(\hat{x}_{i_t} ^T\vartheta_{i_t}^t)) -  \phi_{i_t}'(\hat{x}_{i_t} ^T\theta_{*}))\hat{x}_{i_t}\right\|^2\right]  \\
&&   + (1 + \rho_t)\CE\left[\left\|\nabla f_{i_t}(\theta_{t-1}; \xi_{i_t}) - \nabla f_{i_t}(\theta_*; \xi_{i_t})\right\|^2\right].
 \end{eqnarray*}
Applying this and multiplying by some constant $w$, we obtain
\begin{eqnarray*}
\lefteqn{w\CE[\|\theta_t - \theta_*\|^2]} \nonumber\\
& \leq & w(1 - \eta_t\mu)\|\theta_{t-1} - \theta_*\|^2 - 2w\eta_t\CE[P(\theta_{t}) - P(\theta_*)] \\
&& + 4w(1 + \rho_t^{-1})\eta_t^2\CE\left[\left\|(\phi_{i_t}'(\langle \varphi_{i_t}^t, \vartheta_{i_t}^t\rangle)) -  \phi_{i_t}'(\hat{x}_{i_t} ^T\vartheta_{i_t}^t))\hat{x}_{i_t}\right\|^2\right] + 4w(1 + \rho_t^{-1})\eta_t^2\CE\left[\left\|(\phi_{i_t}'(\hat{x}_{i_t} ^T\vartheta_{i_t}^t)) -  \phi_{i_t}'(\hat{x}_{i_t} ^T\theta_{*}))\hat{x}_{i_t}\right\|^2\right]  
\\ &&  + 2w(1 + \rho_t)\eta_t^2\CE\left[\left\|\nabla f_{i_t}(\theta_{t-1}; \xi_{i_t}) - \nabla f_{i_t}(\theta_*; \xi_{i_t})\right\|^2\right] \\
& \leq & w(1 - \eta_t\mu)\|\theta_{t-1} - \theta_*\|^2 - 2w\eta_t\CE[P(\theta_{t}) - P(\theta_*)]  + 4w(1 + \rho_t^{-1})\eta_t^2\CE\left[\left\|(\phi_{i_t}'(\langle \varphi_{i_t}^t, \vartheta_{i_t}^t\rangle)) -  \phi_{i_t}'(\hat{x}_{i_t} ^T\vartheta_{i_t}^t))\hat{x}_{i_t}\right\|^2\right] \\
&&+ 8w(1 + \rho_t^{-1})\eta_t^2L\left[\frac{1}{n}\sum_{i=1}^nf_i(\vartheta_{i}^t) - F(\theta_*) - \frac{1}{n}\sum_{i=1}^n\langle \nabla f_i(\theta_*), \vartheta_{i}^t - \theta_* \rangle\right].  
\\ &&  + 2w(1 + \rho_t)\eta_t^2\CE\left[\left\|\nabla f_{i_t}(\theta_{t-1}; \xi_{i_t}) - \nabla f_{i_t}(\theta_*; \xi_{i_t})\right\|^2\right].
\end{eqnarray*}
Adding this to (\ref{eq:ct}) yields
\begin{eqnarray*}
\lefteqn{\CE[A_t]} \\
& \leq & (1 - \eta_t\mu)A_{t-1} + (2b + 4w)\eta_t^2(1 + \rho_t^{-1})\CE\left[\left\|(\phi_{i_t}'(\langle \varphi_{i_t}^t, \vartheta_{i_t}^t\rangle)) -  \phi_{i_t}'(\hat{x}_{i_t} ^T\vartheta_{i_t}^t))\hat{x}_{i_t}\right\|^2\right]   - 2w\eta_t\CE[P(\theta_{t}) - P(\theta_*)] \\
&& + (\frac{\alpha_t}{n} - \frac{2b\eta_t(L -\mu)}{L} - 2b\eta_t^2\mu\rho_t)[F(\theta_{t-1}) - F(\theta_*) - \langle \nabla F(\theta_*), \theta_{t-1} - \theta_*\rangle] 
\\ 
&&
+ \eta_t(\eta_tb(1 + \rho_t) +  2\eta_t w(1 + \rho_t)- \frac{b}{L})\CE\left[\left\|\nabla f_{i_t}(\theta_{t-1}; \xi_{i_t}) - \nabla f_{i_t}(\theta_*; \xi_{i_t})\right\|^2\right] 
\\
&& 
+ (\alpha_{t-1}\eta_t\mu + \alpha_t + 4b\eta_t^2L(1 + \rho_t^{-1}) + 8w\eta_t^2L(1 + \rho_t^{-1}) -\alpha_{t-1} - \frac{\alpha_t}{n} )\left[\frac{1}{n}\sum_{i=1}^nf_i(\vartheta_{i}^t) - F(\theta_*) - \frac{1}{n}\sum_{i=1}^n\langle \nabla f_i(\theta_*), \vartheta_{i}^t - \theta_* \rangle\right],
\end{eqnarray*}
where Lyapunov function $A_t$ is defined as  
\begin{eqnarray*}
A_t = (b + w)\|\theta_t - \theta_*\|^2 + \alpha_t\left[\frac{1}{n}\sum_{i=1}^nf_i(\vartheta_{i}^{t+1}) - F(\theta_*) - \frac{1}{n}\sum_{i=1}^n\langle \nabla f_i(\theta_*), \vartheta_{i}^{t+1} - \theta_* \rangle\right].
\end{eqnarray*} 
Let us assume that
\begin{eqnarray*}
\eta_t & = & \frac{c}{\gamma + t} \text{ for some } c > \frac{1}{\mu} \text{ and } \gamma > 0 \text{ such that } \eta_1 \leq \frac{1}{7(\mu n + L)}, \\
\rho_t & = & \frac{\gamma+t}{2cL} -  1, \\
b & = & \frac{1}{2n}, \\
w & = & \frac{1}{4n}, \\
\alpha_t &=& \eta_t(1 - \eta_t\mu - \frac{\mu}{2L}) \text{ and } \alpha_0 = \alpha_1.
\end{eqnarray*}
Then, we obtain
\begin{eqnarray*}
\CE[A_t]
& \leq & (1 - \eta_t\mu)A_{t-1} + (2b + 4w)\eta_t^2(1 + \rho_t^{-1})\CE\left[\left\|(\phi_{i_t}'(\langle \varphi_{i_t}^t, \vartheta_{i_t}^t\rangle)) -  \phi_{i_t}'(\hat{x}_{i_t} ^T\vartheta_{i_t}^t))\hat{x}_{i_t}\right\|^2\right]   - 2w\eta_t\CE[P(\theta_{t}) - P(\theta_*)] \\
& = & (1 - \eta_t\mu)A_{t-1} + \frac{2\eta_t^2(1 + \rho_t^{-1})}{n}\CE\left[\left\|(\phi_{i_t}'(\langle \varphi_{i_t}^t, \vartheta_{i_t}^t\rangle)) -  \phi_{i_t}'(\hat{x}_{i_t} ^T\vartheta_{i_t}^t))\hat{x}_{i_t}\right\|^2\right]   - \frac{\eta_t}{2n}\CE[P(\theta_{t}) - P(\theta_*)] \\
& \leq & (1 - \eta_t\mu)A_{t-1} + \frac{2\eta_t^2}{n}\CE\left[\left\|(\phi_{i_t}'(\langle \varphi_{i_t}^t, \vartheta_{i_t}^t\rangle)) -  \phi_{i_t}'(\hat{x}_{i_t} ^T\vartheta_{i_t}^t))\hat{x}_{i_t}\right\|^2\right] (1 + \frac{2L}{7\mu n  + 5 L})  - \frac{\eta_t}{2n}\CE[P(\theta_{t}) - P(\theta_*)] \\
& \leq & (1 - \eta_t\mu)A_{t-1} + \frac{4\eta_t^2}{n}\CE\left[\left\|(\phi_{i_t}'(\langle \varphi_{i_t}^t, \vartheta_{i_t}^t\rangle)) -  \phi_{i_t}'(\hat{x}_{i_t} ^T\vartheta_{i_t}^t))\hat{x}_{i_t}\right\|^2\right]   - \frac{\eta_t}{2n}\CE[P(\theta_{t}) - P(\theta_*)] \\
\end{eqnarray*}
Let $\eta_t = 2/(\mu(\gamma + t))$. Following similar arguments in the proof of Theorem~\ref{em_ap_conv}, we obtain
\begin{eqnarray*}
\lefteqn{\CE[P(\theta_{t}) - P(\theta_*)]}\\
& \leq & (\frac{1}{\eta_t} - \mu)2n\CE[A_{t-1}] - \frac{2n}{\eta_t}\CE[A_t] +  8\eta_t\sigma_c^2\\
& = & (\gamma + t - 2)n\mu\CE[A_{t-1}] - (\gamma+t)n\mu\CE[A_t]  +  \frac{16\sigma_c^2}{\mu (\gamma + t)}.
\end{eqnarray*}
Multiplying by $\gamma + t - 1$ yields
\begin{eqnarray*}
\lefteqn{(\gamma+t-1)\CE[P(\theta_{t}) - P(\theta_*)]}\\
& \leq & (\gamma+t-1)(\gamma + t - 2)n\mu\CE[A_{t-1}] - (\gamma+t)(\gamma+t-1)n\mu\CE[A_t] + \frac{(\gamma+t-1)}{\gamma + t}\frac{16\sigma_c^2}{\mu}\\
& \leq & (\gamma+t-1)(\gamma + t - 2)n\mu\CE[A_{t-1}] - (\gamma+t)(\gamma+t-1)n\mu\CE[A_t] + \frac{16\sigma_c^2}{\mu}.
\end{eqnarray*}
By summing the above inequality from $t = 1$ to $t= T$, we have a telescoping sum that simplifies as follows:
\begin{eqnarray*}
\lefteqn{\sum_{t=1}^T(\gamma+t-1)\CE[P(\theta_{t}) - P(\theta_*)]}\\
& \leq & \gamma(\gamma - 1)n\mu A_0 - (\gamma+T)(\gamma+T-1)n\mu\CE[A_T]  + \frac{16T\sigma_c^2}{\mu} \\
& \leq & \gamma(\gamma - 1)n\mu A_0   + \frac{16T\sigma_c^2}{\mu}.
\end{eqnarray*}
Dividing by $\sum_{t=1}^T(\gamma+t-1) = (2T\gamma + T(T-1))/2$ and using Jensen's inequality yields
\begin{eqnarray*}
\CE[P(\bar{\theta}_T) - P(\theta_*)]
& \leq & \frac{2n\mu\gamma(\gamma - 1)}{T(2\gamma + T - 1)}A_0  + \frac{32\sigma_c^2}{\mu(2\gamma + T - 1)} \\
& = & \frac{\mu\gamma(\gamma - 1)}{2T(2\gamma + T - 1)}C_4  + \frac{32\sigma_c^2}{\mu(2\gamma + T - 1)},
\end{eqnarray*}
where 
\[C_4 \equiv 3\|\theta_0 - \theta_*\|^2  + \frac{4n}{7(\mu n + L)}\left[F(\theta_0) - F(\theta_*) - \langle \nabla F(\theta_*), \theta_0 - \theta_* \rangle\right].\]

\end{proof}

\end{document}